\documentclass{article}

\PassOptionsToPackage{numbers}{natbib}

\usepackage[final]{neurips_2025}

\usepackage[utf8]{inputenc} 
\usepackage[T1]{fontenc}    
\usepackage{hyperref}       
\usepackage{url}            
\usepackage{booktabs}       
\usepackage{amsfonts}       
\usepackage{nicefrac}       
\usepackage{microtype}      
\usepackage{xcolor}         
\usepackage{graphicx}
\usepackage{enumitem}
\usepackage{amsthm,amssymb,amsmath}
\newtheorem{definition}{Definition}
\newtheorem{proposition}{Proposition}
\usepackage{multirow}
\usepackage{booktabs}
\usepackage{algorithm}
\usepackage{algorithmic}
\usepackage{xspace}
\usepackage{caption}

\newcommand{\ie}{i.e.\xspace}
\newcommand{\eg}{e.g.\xspace}
\usepackage{subfig}

\title{Functional Matching of Logic Subgraphs: \\Beyond Structural Isomorphism}

\author{Ziyang Zheng\quad
Kezhi Li\quad
Zhengyuan Shi\quad
Qiang Xu\quad
\\
The Chinese University of Hong Kong \\
{\small\texttt{\{zyzheng23,kzli24,zyzshi21,qxu\}@cse.cuhk.edu.hk}}\\
}

\begin{document}

\maketitle

\begin{abstract}
Subgraph matching in logic circuits is foundational for numerous Electronic Design Automation (EDA) applications, including datapath optimization, arithmetic verification, and hardware trojan detection. However, existing techniques rely primarily on structural graph isomorphism and thus fail to identify function-related subgraphs when synthesis transformations substantially alter circuit topology. To overcome this critical limitation, we introduce the concept of \emph{functional subgraph matching}, a novel approach that identifies whether a given logic function is implicitly present within a larger circuit, irrespective of structural variations induced by synthesis or technology mapping. Specifically, we propose a two-stage multi-modal framework: (1) learning robust functional embeddings across AIG and post-mapping netlists for functional subgraph detection, and (2) identifying fuzzy boundaries using a graph segmentation approach. Evaluations on standard benchmarks (ITC99, OpenABCD, ForgeEDA) demonstrate significant performance improvements over existing structural methods, with average $93.8\%$ accuracy in functional subgraph detection and a dice score of $91.3\%$ in fuzzy boundary identification. The source code and implementation details can be found at 
\href{https://github.com/zyzheng17/Functional_Subgraph_Matching-Neurips25}{\textit{\textcolor{blue}{our repository}}}.


\end{abstract}

\section{Introduction}
\label{sec:intro}

Subgraph matching—identifying smaller graphs within larger ones—is a fundamental task in graph analysis, with pivotal applications spanning social network mining, bioinformatics, and Electronic Design Automation (EDA). 

In the context of EDA, subgraph matching involves searching for specific circuit patterns embedded within larger circuits. This capability directly supports critical tasks such as circuit optimization, verification, and security analyses. For example, verifying complex arithmetic circuits like multipliers typically requires recognizing embedded small functional units (e.g., half-adders) within larger netlists, enabling algebraic simplifications and correctness proofs~\cite{mahzoon2018polycleaner,mahzoon2019revsca}. Similarly, during template-based synthesis, accurately locating predefined subgraphs allows their replacement with highly optimized standard cells, thereby significantly improving power, performance, and area (PPA) metrics~\cite{wei2015universal}. Moreover, subgraph matching also plays an essential role in hardware security by enabling the identification of potentially malicious substructures or "hardware trojans"—anomalous subcircuits intentionally embedded to compromise system integrity~\cite{meade2016gate,li2019attacking}.

Traditionally, subgraph matching in graphs is formulated as a \emph{structural isomorphism} problem: determining whether a smaller query graph exactly matches part of a larger target graph in terms of node and edge connectivity. This problem is extensively studied in general graph theory, and classical approaches rely primarily on combinatorial search algorithms~\cite{cordella2004sub, ullmann1976algorithm, cordella2001improved}. However, subgraph isomorphism is an NP-complete problem and thus often suffers from exponential computational complexity in worst-case scenarios. Recently, deep learning methods have emerged to mitigate this computational cost by embedding graphs into continuous latent spaces, significantly accelerating matching tasks~\cite{bai2019simgnn, lou2020neural,ying2024representation}. Within the EDA domain, these techniques have been successfully adapted for transistor-level subcircuit identification~\cite{li2024efficient}. 

However, structure-based matching methods encounter significant limitations in practical EDA tasks, as circuit topologies frequently undergo substantial transformations during logic synthesis and technology mapping. Equivalent logic functions can thus be realized through widely differing structural implementations, driven by design considerations such as timing performance, power consumption, or silicon area. Consequently, exact structural correspondence rarely persists throughout the design process, even when the underlying logic function remains unchanged. This inherent limitation severely restricts the utility of traditional structural matching techniques, particularly in applications requiring cross-stage queries—for example, identifying subgraphs from an abstract netlist (like an And-Inverter Graph, or AIG) within a synthesized, technology-mapped netlist.

Motivated by this critical gap, we introduce an approach explicitly designed to recognize logic functionality irrespective of structural differences. Specifically, our framework determines whether the logic represented by a query subgraph exists implicitly within a candidate graph, independent of structural transformations. 

To formalize this, we propose two key concepts: (1) \textbf{functional subgraph}, representing the circuit logic containment relation independent of structure, and (2) \textbf{fuzzy boundary}, minimal graph regions encapsulating the query's logic despite unclear structural boundaries. Consequently, our methodology, termed \textbf{\emph{functional subgraph matching}}, addresses two sub-tasks: \emph{1. Functional Subgraph Detection}: Determining whether the logic function of a query graph is implicitly contained within a candidate graph; \emph{2. Fuzzy Boundary Identification}: Precisely locating the smallest possible region (the fuzzy boundary) in the candidate graph that encapsulates the query's logic.


To achieve these objectives, we propose a novel two-stage multi-modal framework. In the first stage, we train our model with intra-modal and inter-modal alignment across different graph modalities, enabling robust and cross-stage detection of functional subgraph. In the second stage, we fine-tune our model and formulate fuzzy boundary detection as a graph segmentation task, moving beyond prior approaches that treated boundary identification as an input-output classification problem~\cite{wang2022efficient, wu2023gamora}. By leveraging information from nodes located within the true boundaries, our segmentation approach significantly enhances performance and continuity of fuzzy boundary prediction.

Our experiments demonstrate the effectiveness of the proposed framework. Evaluations conducted across several widely-used benchmarks, ITC99~\cite{ITC99}, OpenABCD~\cite{chowdhury2021openabc} and ForgeEDA~\cite{shi2025forgeeda}, show that our approach significantly surpasses traditional structure-based methods. Specifically, our framework achieves an average accuracy of $93.8\%$ for functional subgraph detection and attains a DICE score of $91.3\%$ for fuzzy boundary detection tasks. In contrast, structure-based baseline methods typically exhibit near-random performance (accuracy close to 50\%) and high variability in precision, recall, and F1-score, underscoring their limitations in capturing implicit functionality. To further validate our method's robustness and generalizability, we additionally propose three function-aware baseline variants by integrating different graph encoders into our framework.

In summary, the contributions of this work include:
\begin{itemize}[itemsep=2pt, topsep=0pt, partopsep=0pt]
    \item Introducing and formally defining the novel concept of functional subgraph matching, clearly distinguishing it from structural isomorphism and functional equivalence.
    \item Developing a two-stage multi-modal embedding framework, leveraging both intra-modal and inter-modal alignments to capture structure-agnostic and function-invariant graph representations. This allows effective functional subgraph detection across different modalities.
    \item Proposing an innovative approach for fuzzy boundary identification by formulating the task as a graph segmentation problem rather than a simple input-output classification problem, significantly enhancing boundary continuity and localization accuracy.
\end{itemize}

\begin{figure}[]
    \centering
    \includegraphics[width=0.8\linewidth]{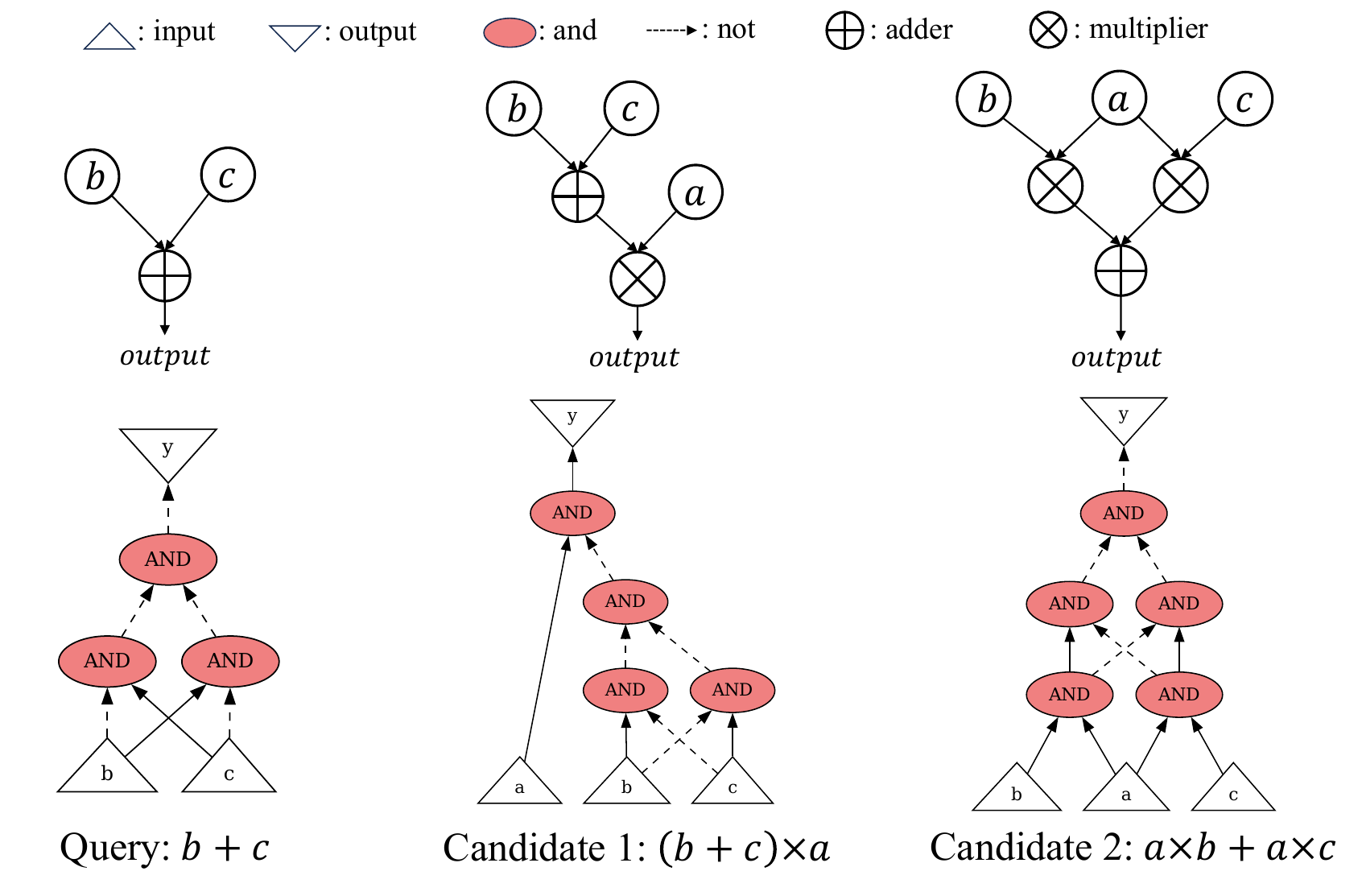}
    \caption{We present the query and candidate graphs.
    \textbf{Above}: 1-bit adder and multiplier. \textbf{Below}: AIG netlist. The query $b+c$ is explicitly contained within the candidate $(b+c) \times a$, making it straightforward to identify the exact subgraph in the candidate. In contrast, the query $b+c$ is implicitly contained within the candidate $a \times b + a \times c$, which implies no subgraph of $a \times b + a \times c$ has the same structure or function as the query graph. }
    \label{fig:toy_equation}
\end{figure}

\vspace{-5pt}
\section{Preliminaries}

\vspace{-2pt}
\subsection{Subgraph Isomorphism Matching}
Subgraph isomorphism matching is a fundamental problem in graph theory with applications across bioinformatics~\cite{bonnici2013subgraph}, social network analysis~\cite{fan2012graph}, and knowledge graphs~\cite{kim2015taming,perez2009semantics}. 
We first recall the standard definition of subgraph isomorphism in Definition~\ref{def:subgraph_iso}.

\begin{definition}[\textbf{Subgraph Isomorphism}]
A graph $\mathcal{Q}$ is an \textbf{isomorphic subgraph} of $\mathcal{G}$ if there exists a subgraph $\mathcal{G}'$ of $\mathcal{G}$ such that $\mathcal{Q}$ is isomorphic to $\mathcal{G}'$.
\label{def:subgraph_iso}
\end{definition}

Then, based on the definition of subgraph isomorphism, the subgraph isomorphism matching task is defined as follows: given a query graph $\mathcal{Q}$ and a target graph $\mathcal{G}$, determine if $\mathcal{Q}$ is isomorphic to a subgraph of $\mathcal{G}$. 
Classical approaches of subgraph isomorphism matching rely primarily on combinatorial search algorithms~\cite{ullmann1976algorithm,cordella2001improved,cordella2004sub}.
Its NP-complete nature, however, makes exact matching computationally intensive.
More recently, graph-neural-network-based (GNN-based) methods have been introduced to learn compact graph embeddings that accelerate the matching process~\cite{bai2019simgnn,lou2020neural,ying2024representation}.
In the EDA domain, \citet{li2024efficient} adapt the NeuroMatch architecture~\cite{lou2020neural} to solve subcircuit isomorphism on transistor-level netlists. 

However, in EDA flow, graphs often represent circuits or computations where structural modifications can preserve the underlying function. Standard subgraph isomorphism struggles with such cases. For instance, as illustrated in Figure~\ref{fig:toy_equation}, a model based on Definition~\ref{def:subgraph_iso} can identify that the structure representing $b+c$ is contained within $a \times (b+c)$, but it cannot identify the functional presence of $b+c$ within the structurally different but functionally related expression $a \times b + a \times c$.

\subsection{Subgraph Equivalence}
The limitation of structure-based subgraph matching motivates considering functional properties. 
Function-aware representation learning has emerged as a pivotal subfield in EDA. Many recent works emphasize functional equivalence, denoted $\mathcal{G}_1 \equiv_{func} \mathcal{G}_2$. DeepGate~\cite{shi2023deepgate2,shi2024deepgate3,zheng2025deepgate4} and DeepCell~\cite{shi2025deepcell} employ disentanglement to produce separate embeddings for functionality and structure, pretraining across various EDA benchmarks and predict functional similarity with a task head. PolarGate~\cite{PolarGate} enhances functional embeddings by integrating ambipolar device principles. FGNN~\cite{wang2022functionality,wang2024fgnn2} applies contrastive learning to align circuit embeddings according to functional similarity.

While graph isomorphism requires structural identity, functional equivalence relates graphs based on their input-output behavior. Building on this, we can define a notion of subgraph relationship based on function, as shown in Definition~\ref{def:subgraph_eq}.

\begin{definition}[\textbf{Subgraph Equivalence}]
A graph $\mathcal{Q}$ is an \textbf{equivalent subgraph} of $\mathcal{G}$ if there exists a subgraph $\mathcal{G}'$ of $\mathcal{G}$ such that $\mathcal{Q} \equiv_{func} \mathcal{G}'$.
\label{def:subgraph_eq}
\end{definition}

This definition allows for functional matching within existing subgraphs. Some works adopt similar ideas for tasks such as arithmetic block identification~\cite{wang2022efficient,he2021graph} and symbolic reasoning~\cite{wu2023gamora,deng2024less}, which aim to find a subgraph with specific functionality rather than structure. Compared to subgraph isomorphism, subgraph equivalence offers more flexibility against local structure modifications. However, Definition~\ref{def:subgraph_eq} still falls short for cases involving global restructuring. As shown in Figure~\ref{fig:toy_equation}, in the example $a \times b + a \times c$, no single subgraph is functionally equivalent to $b+c$. The function $b+c$ is \textbf{implicitly} present but not explicitly represented by a contiguous subgraph.

\subsection{Functional Subgraph}
\label{sec:property}
To address the limitations of both Definition~\ref{def:subgraph_iso} and Definition~\ref{def:subgraph_eq}, we introduce the concept of a functional subgraph, which aims to identify the implicit containment relation between graphs.

\begin{definition}[\textbf{Functional Subgraph}]
A graph $\mathcal{Q}$ is a \textbf{functional subgraph} of $\mathcal{G}$, denoted $\mathcal{Q} \preccurlyeq \mathcal{G}$, if there exists a graph $\mathcal{G}'$ such that $\mathcal{G}' \equiv_{func} \mathcal{G}$ and $\mathcal{Q}$ is isomorphic to a subgraph of $\mathcal{G}'$.
\label{def:subgraph_func}
\end{definition}

This definition captures the idea that the query's function is implicitly contained within the target's function, 
even if the target's structure has undergone functional transformations, and no exact subgraph isomorphic to the query graph can be found in the target graph. By this definition, we know that $b+c$ is a functional subgraph of $a \times b + a \times c$ since $a \times b + a \times c \equiv_{func}a \times (b+c)$ and $b+c$ is an isomorphic subgraph of $a \times (b+c)$.
Furthermore, Definition~\ref{def:subgraph_func} encompasses Definition~\ref{def:subgraph_eq}, i.e., Definition~\ref{def:subgraph_eq} is a special case of Definition~\ref{def:subgraph_func}, as discussed in Proposition~\ref{pro:eq_func} (proof in Appendix~\ref{apdx:property}).
\begin{proposition}
If a graph $\mathcal{Q}$ is an equivalent subgraph of $\mathcal{G}$, then $\mathcal{Q}$ is a functional subgraph of $\mathcal{G}$.
\label{pro:eq_func}
\end{proposition}

\paragraph{Properties of Functional Subgraph}
In this paper, we assume that a graph obtained by removing some nodes and edges is not functionally equivalent to the original graph, \ie $\forall g \neq \emptyset$, $G \setminus g \not\equiv_{\mathrm{func}} G$.
For example, we consider it illegal to directly connect two NOT gates. Therefore, such connections do not appear in our graph structures. In fact, EDA tools such as ABC~\cite{brayton2010abc} inherently enforce this constraint. According to Definition~\ref{def:subgraph_func}, functional subgraphs exhibit the following properties:
\begin{itemize}
    \item \textit{\textbf{Reflexivity}}: For any graph $\mathcal{G}$, $\mathcal{G}$ is the functional subgraph of $\mathcal{G}$, \ie $\forall \mathcal{G}, \mathcal{G}\preccurlyeq \mathcal{G}$.
    \item \textit{\textbf{Functional Equivalence Preservation}}: If $\mathcal{G}_1$ is a functional subgraph of $\mathcal{G}_2$, then:
        \begin{itemize}
            \item (Left-hand Side) if $\mathcal{G}'_1$ is functionally equivalent to $\mathcal{G}_1$, then $\mathcal{G}'_1$ is a functional subgraph of $\mathcal{G}_2$, \ie if $\mathcal{G}_1\preccurlyeq\mathcal{G}_2$ and $\mathcal{G}'_1 \equiv_{func} \mathcal{G}_1$, then $\mathcal{G}'_1 \preccurlyeq \mathcal{G}_2$.
            \item (Right-hand Side) if $\mathcal{G}'_2$ is functionally equivalent to $\mathcal{G}_2$, then $\mathcal{G}_1$ is a functional subgraph of $\mathcal{G}'_2$, \ie if $\mathcal{G}_1\preccurlyeq\mathcal{G}_2$ and $\mathcal{G}'_2 \equiv_{func} \mathcal{G}_2$, then $\mathcal{G}_1 \preccurlyeq \mathcal{G}'_2$.
        \end{itemize}
        


    \item \textit{\textbf{Transitivity}}: If $\mathcal{G}_1$ is a functional subgraph of $\mathcal{G}_2$ and $\mathcal{G}_2$ is a functional subgraph of $\mathcal{G}_3$, then $\mathcal{G}_1$ is a functional subgraph of $\mathcal{G}_3$, \ie if $\mathcal{G}_1\preccurlyeq\mathcal{G}_2$ and $\mathcal{G}_2 \preccurlyeq \mathcal{G}_3$, then $\mathcal{G}_1 \preccurlyeq \mathcal{G}_3$.
    \item \textit{\textbf{Anti-symmetry}}: If $\mathcal{G}_1$ is a functional subgraph of $\mathcal{G}_2$, then $\mathcal{G}_2$ is a functional subgraph of $\mathcal{G}_1$ if and only if they are functionally equivalent, \ie $\mathcal{G}_1\preccurlyeq\mathcal{G}_2$ and $\mathcal{G}_2\preccurlyeq\mathcal{G}_1$ if and only if $\mathcal{G}_1\equiv_{func}\mathcal{G}_2$.
\end{itemize}

For detailed proofs of the above properties, please refer to Appendix~\ref{apdx:property}. It is worth noting that the subgraph equivalence defined in Definition~\ref{def:subgraph_eq} does \textbf{not} satisfy the \textit{Transitivity} property. This highlights the improved completeness of the functional subgraph in Definition~\ref{def:subgraph_func}.

\subsection{Task Definition}

Based on Definition~\ref{def:subgraph_func}, we define our primary task:
\paragraph{Task \#1: Functional Subgraph Detection.} Given a query graph $\mathcal{Q}$ and a candidate graph $\mathcal{G}$, determine if $\mathcal{Q} \preccurlyeq \mathcal{G}$.

While functional subgraph detection is a decision problem (yes/no), it is often desirable to identify \textbf{which part} of the target graph $\mathcal{G}$ corresponds to the query function $\mathcal{Q}$. However, as shown in Figure~\ref{fig:toy_equation}, due to potential functional transformations, identifying an exact boundary in the original graph $\mathcal{G}$ that perfectly represents $\mathcal{Q}$ can be challenging or impossible. This leads to our second task, which aims to find the smallest region in $\mathcal{G}$ that encapsulates the function of $\mathcal{Q}$.

\begin{definition}[\textbf{Fuzzy Boundary}]
Given a query graph $\mathcal{Q}$ and a candidate graph $\mathcal{G}=(V,E)$, a subgraph $\mathcal{G}^*=(V^*, E^*)$ of $\mathcal{G}$, where $V^* \subseteq V$ and $E^* = E \cap (V^* \times V^*)$, is a \textbf{fuzzy boundary} for $\mathcal{Q}$ in $\mathcal{G}$ if:
\begin{enumerate}
    \item $\mathcal{Q} \preccurlyeq \mathcal{G}^*$ 
    \item For any proper subgraph $\mathcal{H}$ of $\mathcal{G}^*$ (i.e., $\mathcal{H} \subset \mathcal{G}^*$ and $\mathcal{H} \neq \mathcal{G}^*$), $\mathcal{Q} \not\preccurlyeq \mathcal{H}$ 
\end{enumerate}
\label{def:boundary}
\end{definition}

As illustrated in Figure~\ref{fig:toy_equation}, for $\mathcal{G}$ representing $a \times b + a \times c$ and $\mathcal{Q}$ representing $b+c$, the fuzzy boundary $\mathcal{G}^*$ would likely encompass the components corresponding to $b$, $c$, the two multiplications, and the addition, as this minimal collection is required to functionally contain $b+c$ via transformation. Based on Definition~\ref{def:boundary}, we further define another task as:


\paragraph{Task \#2: Fuzzy Boundary Identification.} Given a query graph $\mathcal{Q}$ and a candidate graph $\mathcal{G}$ such that $\mathcal{Q} \preccurlyeq \mathcal{G}$, determine for each node in $\mathcal{G}$, whether it belongs to the fuzzy boundary $\mathcal{G}^*$ of $\mathcal{Q}$.

\section{Method}
\label{sec:method}

\subsection{Stage \#1: Functional Subgraph Detection}
\begin{figure}[h]
    \vspace{-5pt}
    \centering
    \includegraphics[width=1.0\linewidth]{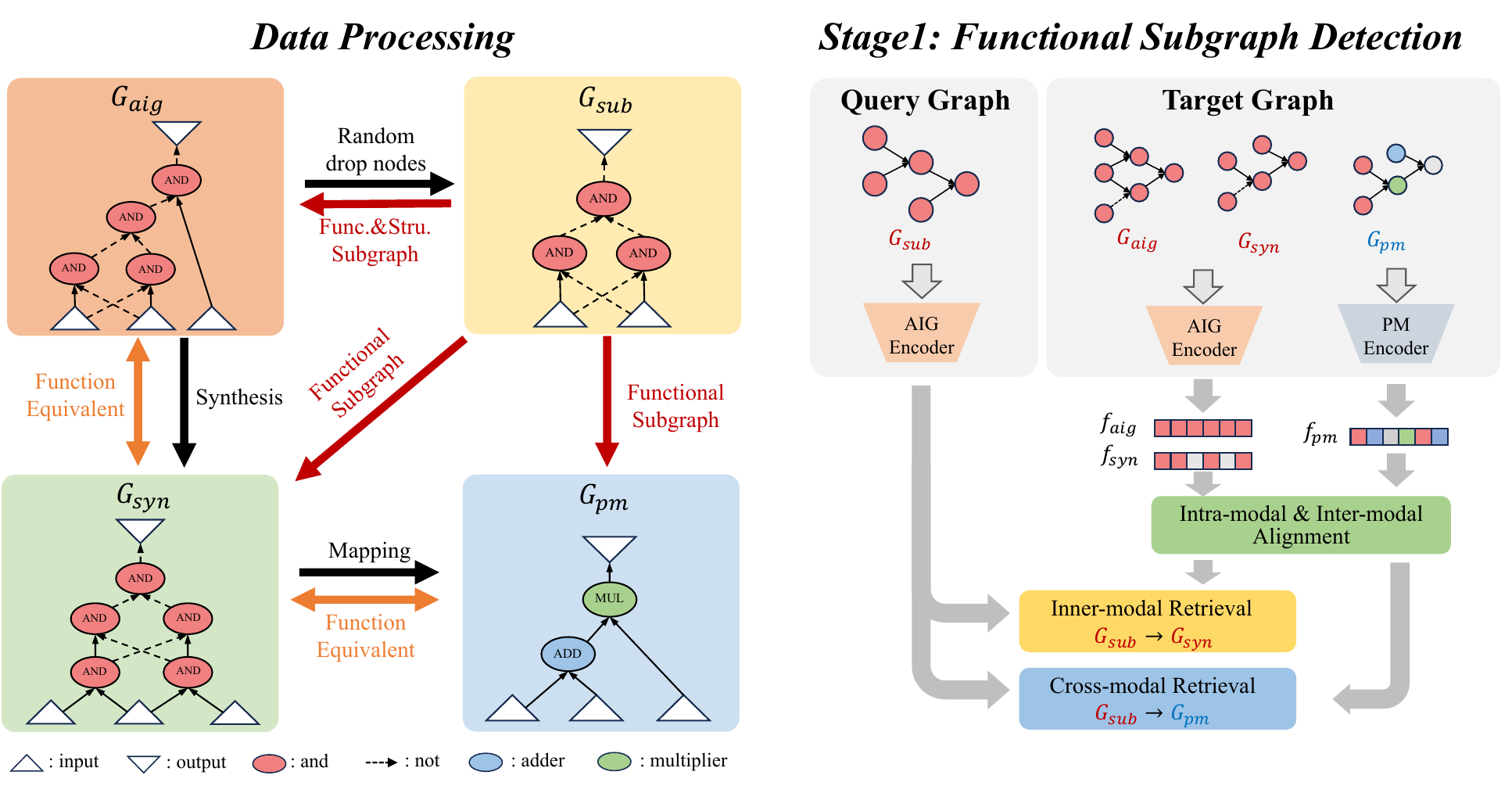}
    \caption{The pipeline of Stage \#1. \textbf{Left}: Our data processing pipeline. For a given $\mathcal{G}_{aig}$, we first randomly extract a subgraph $\mathcal{G}_{sub}$. Then, we obtain $\mathcal{G}_{syn}$ and $\mathcal{G}_{pm}$ through synthesis and mapping, respectively. \textbf{Right}: Our training pipeline via intra-modal and inter-modal alignments for functional subgraph detection. We first encode the query and target graphs using their respective encoders. Next, we perform intra-modal and inter-modal alignment on the target graph to obtain function-invariant and structure-agnostic embeddings. These embeddings are then sent to a task head to determine whether the query graph is contained within the target graph.}
    \label{fig:stage1}
    \vspace{-10pt}
\end{figure}

\paragraph{Data Processing} As illustrated in Figure~\ref{fig:stage1}, given an AIG netlist $\mathcal{G}_{aig}$, we first randomly drop nodes while ensuring legality, to obtain the subgraph $\mathcal{G}_{sub}$. Next, we use the ABC tool~\cite{brayton2010abc} to generate $\mathcal{G}_{syn}$ by randomly selecting a synthesis flow. Importantly, in this step we ensure that $\mathcal{G}_{syn}$ is not isomorphic to $\mathcal{G}_{aig}$. Finally, we apply the ABC tool again to map $\mathcal{G}_{syn}$ to $\mathcal{G}_{pm}$ using the Skywater Open Source PDK~\cite{sky130nm}. This data processing pipeline ensures that $\mathcal{G}_{aig}$ is equivalent to both $\mathcal{G}_{syn}$ and $\mathcal{G}_{pm}$. Since $\mathcal{G}_{sub}$ is an isomorphic subgraph of $\mathcal{G}_{aig}$, it follows from Definition~\ref{def:subgraph_func} that $\mathcal{G}_{sub}$ is a functional subgraph of both $\mathcal{G}_{syn}$ and $\mathcal{G}_{pm}$. For negative pairs, following the approach in \citet{li2024efficient}, we randomly sample $\mathcal{G}_{aig}$, $\mathcal{G}_{syn}$, and $\mathcal{G}_{pm}$ from other pairs within the same batch. It is important to note that all circuits in this paper have multiple inputs and a single output.
For more details, please refer to Section~\ref{sec:data} and Appendix~\ref{apdx:dataset}.

\begin{figure}[]
    \centering
    \includegraphics[width=1.0\linewidth]{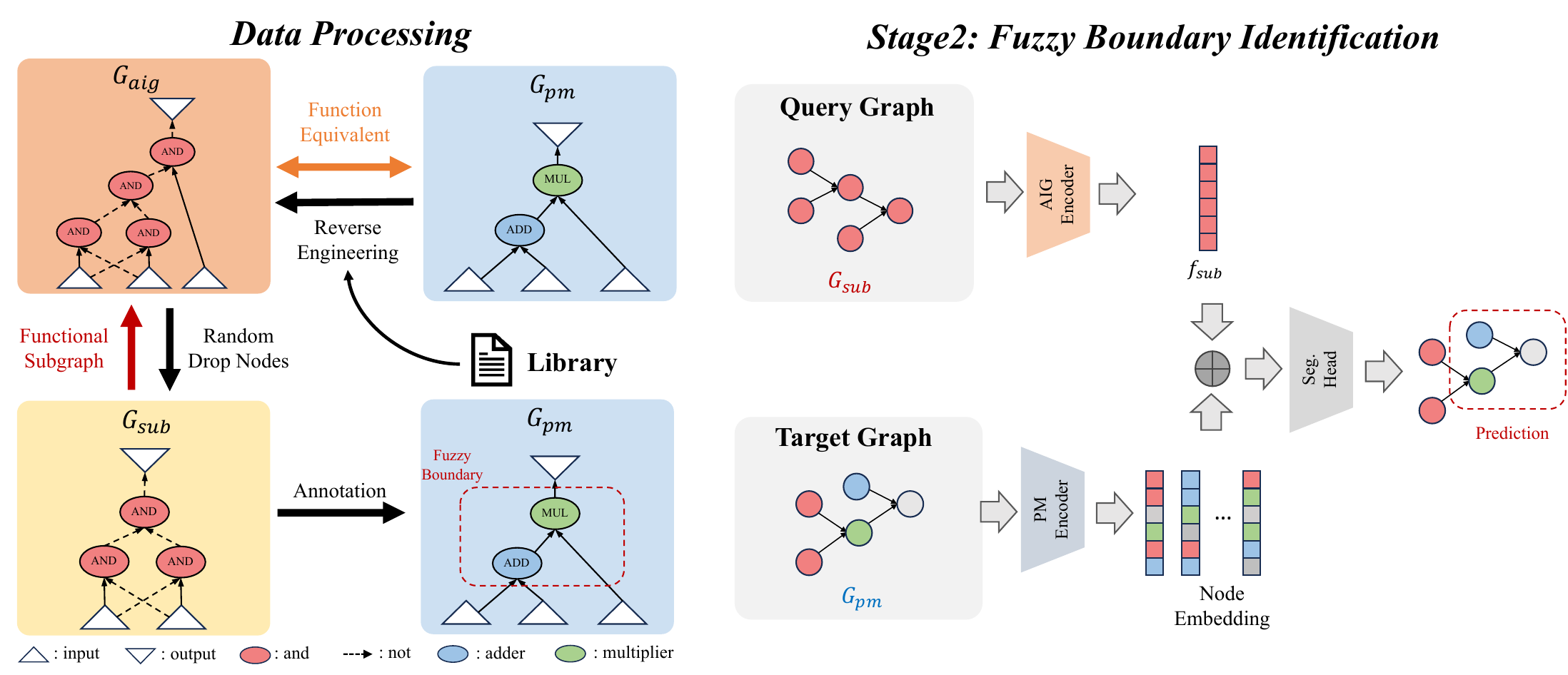}
    \caption{The pipeline of Stage \#2. \textbf{Left}: Our data processing pipeline. For a given $\mathcal{G}_{pm}$, we replace each node in $\mathcal{G}_{pm}$ with the AIG implementation according to the functionality in the library. Then, we randomly sample a subgraph $\mathcal{G}_{sub}$ from $\mathcal{G}_{aig}$. Finally, we annotate the nodes in $\mathcal{G}_{pm}$ if one of the corresponding AIG nodes still exist in $\mathcal{G}_{sub}$. \textbf{Right}: Our training pipeline for fuzzy boundary identification via graph segmentation. Given the query graph $\mathcal{G}_{sub}$ and the target graph $\mathcal{G}_{pm}$, we first use $Enc_{aig}$ to obtain the graph embedding of $\mathcal{G}_{sub}$ and $Enc_{pm}$ to obtain the node embeddings of $\mathcal{G}_{pm}$. These embeddings are then concatenated and passed to a task head to determine whether a node in $\mathcal{G}_{pm}$ lies within the fuzzy boundary of $\mathcal{G}_{sub}$.}
    \label{fig:stage2}
    \vspace{-10pt}
\end{figure}
\paragraph{Retrieval} In this paper, we adopt DeepGate2~\cite{shi2023deepgate2} and DeepCell~\cite{shi2025deepcell} as backbones for encoding AIG netlists and post-mapping netlists, respectively. Given a query graph $\mathcal{G}_{sub}$, along with positive candidates $\mathcal{G}_{aig}^{+}, \mathcal{G}_{syn}^{+}, \mathcal{G}_{pm}^{+}$ and negative candidates $\mathcal{G}_{aig}^{-}, \mathcal{G}_{syn}^{-}, \mathcal{G}_{pm}^{-}$, we first use the AIG encoder $Enc_{aig}$ and the PM encoder $Enc_{pm}$ for different modalities as follows:
\begin{align*}
    f_{sub} &= Enc_{aig}(\mathcal{G}_{sub}),\ 
    f_{aig} = Enc_{aig}(\mathcal{G}_{aig}) \\
    f_{syn} &= Enc_{aig}(\mathcal{G}_{syn}),\ 
    f_{pm} = Enc_{pm}(\mathcal{G}_{pm})
\end{align*}
Next, we concatenate the embeddings of the query graph and the candidate graphs and feed them into a classification head, a 3-layer MLP:
\begin{align*}
    \hat{y}_{aig} &= MLP([f_{sub}, f_{aig}]),\ 
    \hat{y}_{syn} = MLP([f_{sub}, f_{syn}]),\ 
    \hat{y}_{pm} = MLP([f_{sub}, f_{pm}])
\end{align*}
Finally, we compute the binary cross-entropy (BCE) loss for each prediction:
\begin{align*}
    L_{cls} = BCELoss(\hat{y}_{aig}, y_{aig}) + BCELoss(\hat{y}_{syn}, y_{syn}) + BCELoss(\hat{y}_{pm}, y_{pm})
\end{align*}

\paragraph{Function-Invariant Alignment} EDA flows such as synthesis and mapping modify the circuit structure while preserving functional equivalence. As defined in Definition~\ref{def:subgraph_func}, the functional subgraph relation focuses on the functionality of the candidate circuits rather than structure, as they can be transformed into an equivalent circuit with any structure. Furthermore, the \textit{Functional Equivalence Preservation} property in Section~\ref{sec:property} imply that, if the subgraph relation $Q\preccurlyeq G$ hold, then if we replace $Q$ or $G$ with another functional equivalent graph, the subgraph relation continues to hold. This invariance is the key insight for our alignment: the embeddings of functionally equivalent graphs should be aligned, regardless of their structural variations.

Therefore, learning function-invariant embeddings for equivalent circuits across different stages is crucial for functional subgraph detection. 
While $\mathcal{G}_{aig}$ and $\mathcal{G}_{syn}$ share the same gate types, $\mathcal{G}_{aig}$ and $\mathcal{G}_{pm}$ differ significantly in modality, \ie, the gate types in $\mathcal{G}_{pm}$ are substantially dissimilar to those in $\mathcal{G}_{aig}$.
Therefore, we employ both intra-modal and inter-modal alignment techniques to acquire function-invariant and structure-agnostic embeddings with the InfoNCE loss~\cite{oord2018representation}. We select $\mathcal{G}_{aig}$ as the anchor and compute the intra-modal and inter-modal losses as follows:
\begin{align*}
    L_{intra} &= InfoNCE(f_{aig}^{+}, f_{syn}^{+}, f_{syn}^{-}) \\
    L_{inter} &= InfoNCE(f_{aig}^{+}, f_{pm}^{+}, f_{pm}^{-})
\end{align*}
Finally, we summarize the losses for stage \#1 as:
\begin{align*}
    L_{stage_1} = L_{cls} + L_{intra} + L_{inter}
\end{align*}

\subsection{Stage \#2: Fuzzy Boundary Identification}

\paragraph{Data Processing} As illustrated in Figure~\ref{fig:stage2}, given a post-mapping netlist $\mathcal{G}_{pm}$, we replace the cells in $\mathcal{G}_{pm}$ with the corresponding AIGs from the library to acquire the netlist $\mathcal{G}_{aig}$. This process yields a mapping function $\phi$ that associates the node indices of $\mathcal{G}_{aig}$ with those of $\mathcal{G}_{pm}$. Next, we randomly drop nodes to obtain $\mathcal{G}_{sub}$, which serves as the functional subgraph of both $\mathcal{G}_{pm}$ and $\mathcal{G}_{aig}$. Using the subgraph $\mathcal{G}_{sub}$, we annotate the nodes in $\mathcal{G}_{pm}$ by mapping the node indices of $\mathcal{G}_{sub}$ to those of $\mathcal{G}_{pm}$ through the function $\phi$. Specifically, for each node in $\mathcal{G}_{sub}$, if it maps to a node $i$ in $\mathcal{G}_{pm}$, we annotate node $i$ as 1; otherwise, we annotate it as 0. This annotation process strictly follows the fuzzy boundary definition in Definition~\ref{def:boundary}. 

\vspace{-5pt}
\paragraph{Cross-modal Retrieval} Given a query graph $\mathcal{G}_{sub}$ and a target graph $\mathcal{G}_{pm} = (V_{pm}, E_{pm})$, we first compute the embedding of $\mathcal{G}_{sub}$ and the node embeddings of $\mathcal{G}_{pm}$:
\begin{align*}
    f_{sub} = Enc_{aig}(\mathcal{G}_{sub}), \ 
    f_{pm}^{1}, f_{pm}^{2}, \dots, f_{pm}^{|V_{pm}|} = Enc_{pm}(\mathcal{G}_{pm})
\end{align*}
Next, we use $f_{sub}$ as the query embedding and concatenate it with the node embeddings from $\mathcal{G}_{pm}$. These concatenated embeddings are then fed into a 3-layer MLP for node classification: $\hat{y}_i = MLP([f_{sub}, f_{pm}^{i}])$.
While previous works~\cite{wang2022efficient, he2021graph} treat this task as an input-output classification problem, we frame it as a graph segmentation task. This approach arises from the observation that nodes near the input-output nodes contribute to identifying fuzzy boundaries and thus should not be simply labeled as zero. During training, we optimize the model using cross-entropy loss:
\begin{equation}
L_{stage_2} = - \sum_{i} \left[ y_i \log(\hat{y}_i) + (1 - y_i) \log(1 - \hat{y}_i) \right]
\end{equation}

\section{Experiment}

\subsection{Experimental Setup}
\label{sec:data}

\begin{table}[b]
\vspace{-15pt}
\caption{Result of Functional Subgraph Detection(\%).}
\footnotesize
\setlength{\tabcolsep}{2pt}
\begin{tabular}{@{}cccccccccc@{}}
\toprule
\multirow{2}{*}{Dataset} & \multirow{2}{*}{Method} & \multicolumn{4}{c}{$\mathcal{G}_{sub}\to\mathcal{G}_{syn}$} & \multicolumn{4}{c}{$\mathcal{G}_{sub} \to \mathcal{G}_{pm}$} \\ \cmidrule(l){3-10} 
 &  & Accuracy & Precision & Recall & F1-score & Accuracy & Precision & Recall & F1-score \\ \midrule
\multirow{5}{*}{$\rotatebox{90}{ITC99}$} & NeuroMatch & $49.8_{\pm 0.3}$ & $16.7_{\pm 23.6}$ & $33.3_{\pm 47.1}$ & $22.2_{\pm 31.4}$ & $49.8_{\pm 0.2}$ & $16.7_{\pm 23.6}$ & $50.0_{\pm 50.0}$ & $33.4_{\pm 33.4}$ \\
 & HGCN & $44.5_{\pm 7.7}$ & $35.0_{\pm 21.2}$ & $67.3_{\pm 46.3}$ & $45.3_{\pm 30.2}$ & $49.5_{\pm 0.8}$ & $35.7_{\pm 20.2}$ & $66.8_{\pm 47.0}$ & $44.7_{\pm 31.2}$ \\
 & Gamora & $50.6_{\pm 12.8}$ & $21.1_{\pm 27.7}$ & $33.0_{\pm 46.0}$ & $25.4_{\pm 34.8}$ & $51.7_{\pm 4.4}$ & $34.0_{\pm 24.2}$ & $51.2_{\pm 40.9}$ & $40.2_{\pm 29.6}$ \\
 & ABGNN & $56.4_{\pm 9.1}$ & $20.8_{\pm 29.4}$ & $32.7_{\pm 46.3}$ & $25.4_{\pm 35.9}$ & $54.1_{\pm 5.8}$ & $19.0_{\pm 26.9}$ & $33.3_{\pm 47.1}$ & $24.2_{\pm 34.2}$ \\ \cmidrule(l){2-10} 
 & Ours & $\mathbf{95.3_{\pm 0.1}}$ & $\mathbf{94.4_{\pm 0.2}}$ & $\mathbf{96.3_{\pm 0.1}}$ & $\mathbf{95.4_{\pm 0.0}}$ & $\mathbf{93.1_{\pm 0.3}}$ & $\mathbf{92.3_{\pm 0.3}}$ & $\mathbf{94.2_{\pm 0.9}}$ & $\mathbf{93.2_{\pm 0.4}}$ \\ \midrule
\multirow{5}{*}{$\rotatebox{90}{OpenABCD}$} & NeuroMatch & $44.2_{\pm 9.8}$ & $17.3_{\pm 23.9}$ & $33.4_{\pm 47.1}$ & $22.7_{\pm 31.8}$ & $44.9_{\pm 8.4}$ & $17.0_{\pm 23.9}$ & $33.4_{\pm 47.1}$ & $22.5_{\pm 31.7}$ \\
 & HGCN & $52.5_{\pm 3.6}$ & $18.0_{\pm 25.5}$ & $32.5_{\pm 46.0}$ & $23.2_{\pm 32.8}$ & $50.0_{\pm 0.0}$ & $20.4_{\pm 21.4}$ & $33.0_{\pm 46.7}$ & $22.2_{\pm 31.3}$ \\
 & Gamora & $50.8_{\pm 1.1}$ & $33.7_{\pm 23.9}$ & $66.6_{\pm 47.1}$ & $44.8_{\pm 31.7}$ & $49.8_{\pm 0.3}$ & $33.2_{\pm 23.5}$ & $62.1_{\pm 44.3}$ & $43.3_{\pm 30.6}$ \\
 & ABGNN & $34.1_{\pm 5.4}$ & $5.2_{\pm 3.9}$ & $2.6_{\pm 2.6}$ & $3.4_{\pm 3.2}$ & $41.3_{\pm 4.0}$ & $9.7_{\pm 7.6}$ & $3.5_{\pm 3.2}$ & $5.1_{\pm 4.5}$ \\ \cmidrule(l){2-10} 
 & Ours & $\mathbf{92.3_{\pm 0.2}}$ & $\mathbf{93.7_{\pm 0.2}}$ & $\mathbf{90.6_{\pm 0.4}}$ & $\mathbf{92.1_{\pm 0.2}}$ & $\mathbf{90.8_{\pm 0.4}}$ & $\mathbf{92.4_{\pm 0.4}}$ & $\mathbf{88.9_{\pm 0.9}}$ & $\mathbf{90.6_{\pm 0.5}}$ \\ \midrule
\multirow{5}{*}{$\rotatebox{90}{ForgeEDA}$} & NeuroMatch & $50.0_{\pm 0.0}$ & $16.7_{\pm 23.6}$ & $33.3_{\pm 47.1}$ & $22.2_{\pm 31.4}$ & $50.0_{\pm 0.0}$ & $16.7_{\pm 23.6}$ & $33.3_{\pm 47.1}$ & $22.2_{\pm 31.4}$ \\
 & HGCN & $44.0_{\pm 8.5}$ & $18.2_{\pm 22.6}$ & $33.9_{\pm 46.7}$ & $23.1_{\pm 30.8}$ & $48.8_{\pm 1.6}$ & $18.8_{\pm 22.2}$ & $33.5_{\pm 47.0}$ & $22.5_{\pm 31.2}$ \\
 & Gamora & $40.6_{\pm 6.3}$ & $2.4_{\pm 1.6}$ & $0.7_{\pm 0.8}$ & $1.0_{\pm 1.1}$ & $48.2_{\pm 1.5}$ & $51.0_{\pm 8.2}$ & $31.0_{\pm 31.6}$ & $28.5_{\pm 22.9}$ \\
 & ABGNN & $52.3_{\pm 3.3}$ & $34.6_{\pm 24.5}$ & $66.6_{\pm 47.1}$ & $45.5_{\pm 32.2}$ & $52.0_{\pm 2.9}$ & $34.4_{\pm 24.4}$ & $66.6_{\pm 47.1}$ & $45.4_{\pm 32.1}$ \\ \cmidrule(l){2-10} 
 & Ours & $\mathbf{96.0_{\pm 0.1}}$ & $\mathbf{96.8_{\pm 0.4}}$ & $\mathbf{95.2_{\pm 0.5}}$ & $\mathbf{96.0_{\pm 0.1}}$ & $\mathbf{95.3_{\pm 0.0}}$ & $\mathbf{95.9_{\pm 0.5}}$ & $\mathbf{94.7_{\pm 0.5}}$ & $\mathbf{95.3_{\pm 0.0}}$ \\ \bottomrule
\end{tabular}
\label{tab:stage1}
\end{table}

We evaluate our method on three AIG datasets: ITC99~\cite{ITC99}, OpenABCD~\cite{chowdhury2021openabc}, and ForgeEDA~\cite{shi2025forgeeda}. 
Each metric in Tables~\ref{tab:stage1} and ~\ref{tab:stage2} is reported as the \textit{mean $\pm$ standard deviation} over three independent runs.
For data processing, we begin by randomly sampling $k$-hop subgraphs (with $k$ ranging from 8 to 12) to partition large circuits into smaller circuits. 
Next, we randomly sample subgraphs from these smaller circuits.
For logic synthesis, we use the ABC tool~\cite{brayton2010abc} with a randomly selected flow from \textit{src\_rw}, \textit{src\_rs}, \textit{src\_rws}, \textit{resyn2rs}, and \textit{compress2rs}. We then apply the VF2 algorithm~\cite{cordella2004sub} to verify that the synthesis process has modified the circuit structure. If no modification is detected, we repeat this step until we obtain a circuit with a different structure. For technology mapping, we invoke ABC with the Skywater Open Source PDK~\cite{sky130nm}. For additional details on the environment, evaluation metrics, and dataset statistics, please refer to Appendix~\ref{apdx:dataset}.

\subsection{Stage \#1: Functional Subgraph Detection}

We evaluate the performance of our proposed method on three datasets: ITC99, OpenABCD, and ForgeEDA. Our method is compared against several state-of-the-art models, including NeuroMatch~\cite{lou2020neural} and HGCN~\cite{li2024efficient}, which are designed for isomorphism subgraph matching in general domain and EDA domain respectively, and Gamora~\cite{wu2023gamora} and ABGNN~\cite{wang2022efficient}, which are designed for reasoning in EDA domain, \ie for equivalent subgraph matching. Since Gamora and ABGNN focus on boundary detection instead of subgraph matching, we integrate them into the NeuroMatch framework for Stage \#1. Further integration of Gamora and ABGNN with our method is discussed in Appendix~\ref{apdx:other_baseline}. The evaluation metrics include accuracy, precision, recall, and F1-score, computed for two tasks: $\mathcal{G}_{sub}$ to $\mathcal{G}_{syn}$ and $\mathcal{G}_{sub}$ to $\mathcal{G}_{pm}$.

As shown in Table~\ref{tab:stage1}, the results on the ITC99, OpenABCD, and ForgeEDA datasets demonstrate that our method significantly outperforms all baseline models. Specifically, for the $\mathcal{G}_{sub} \to \mathcal{G}_{syn}$ task, our model achieves an average accuracy of $\mathbf{94.5\%}$, precision of $\mathbf{95.0\%}$, recall of $\mathbf{94.0\%}$, and F1-score of $\mathbf{94.5\%}$, surpassing all other methods by a large margin. Similarly, for the $\mathcal{G}_{sub} \to \mathcal{G}_{pm}$ task, our method also shows superior performance with an accuracy of $\mathbf{93.1\%}$, precision of $\mathbf{93.5\%}$, recall of $\mathbf{92.6\%}$, and F1-score of $\mathbf{93.0\%}$. In contrast, structure-based methods show an accuracy close to 50\% and large standard errors in precision, recall, and F1-score. Such unreliable performance typically arises because these methods indiscriminately predict all pairs as either entirely positive or negative, highlighting their limitations in functional subgraph detection.

\vspace{-5pt}
\subsection{Stage \#2: Fuzzy Boundary Identification}

\begin{table}[h]
\centering
\small
\vspace{-10pt}
\caption{Result of Fuzzy Boundary Identification(\%).}
\begin{tabular}{@{}ccccccc@{}}
\toprule
\multirow{2}{*}{Method} & \multicolumn{2}{c}{ITC99} & \multicolumn{2}{c}{OpenABCD} & \multicolumn{2}{c}{ForgeEDA} \\ \cmidrule(l){2-7} 
 & IoU & DICE & IoU & DICE & IoU & DICE \\ \midrule
NeuroMatch & $44.2_{\pm 0.0}$ & $61.3_{\pm 0.0}$ & $41.2_{\pm 0.0}$ & $58.3_{\pm 0.0}$ & $42.0_{\pm 0.0}$ & $59.1_{\pm 0.0}$ \\
HGCN & $44.1_{\pm 0.0}$ & $61.2_{\pm 0.0}$ & $41.2_{\pm 0.0}$ & $58.3_{\pm 0.0}$ & $42.0_{\pm 0.0}$ & $59.2_{\pm 0.0}$ \\
Gamora & $39.1_{\pm 2.8}$ & $56.2_{\pm 2.9}$ & $44.2_{\pm 1.2}$ & $61.3_{\pm 1.1}$ & $39.5_{\pm 0.6}$ & $56.6_{\pm 0.6}$ \\
ABGNN & $26.7_{\pm 6.2}$ & $41.7_{\pm 7.5}$ & $37.5_{\pm 0.8}$ & $54.5_{\pm 0.8}$ & $31.9_{\pm 2.6}$ & $48.2_{\pm 3.0}$ \\ \midrule
Ours & $\mathbf{83.0_{\pm   1.4}}$ & $\mathbf{90.7_{\pm 0.9}}$ & $\mathbf{85.2_{\pm 0.9}}$ & $\mathbf{92.0_{\pm 0.5}}$ & $\mathbf{83.8_{\pm 0.8}}$ & $\mathbf{91.2_{\pm 0.4}}$ \\ \bottomrule
\end{tabular}
\label{tab:stage2}
\end{table}

In this stage, we treat $\mathcal{G}_{sub}$ as the query and aim to locate its fuzzy boundary within the post-mapping netlist $\mathcal{G}_{pm}$. Since Gamora and ABGNN are designed for the detection of the input-output boundary, we first apply each to identify the input and output nodes in $\mathcal{G}_{pm}$. We then perform a BFS between inputs and outputs to recover the corresponding fuzzy boundary, and evaluate the result using Intersection-over-Union (IoU) and DICE score.  

Table~\ref{tab:stage2} reports results on ITC99, OpenABCD, and ForgeEDA, demonstrating that our model substantially outperforms all baselines. Specifically, we achieve an average IoU of $\mathbf{84.0\%}$ and a Dice score of $\mathbf{91.3\%}$, significantly outperforming all other methods. Structure-based methods (e.g., NeuroMatch and HGCN) fail to capture functional boundaries and often generate trivial solutions (predicting all nodes as boundary nodes), yielding low variance but poor performance. Although Gamora and ABGNN can detect clear block boundaries for specific arithmetic modules, they struggle with the variable, function-driven fuzzy boundaries required here, resulting in significantly lower performance. Further integration of Gamora and ABGNN within our framework is detailed in Appendix~\ref{apdx:other_baseline}.

\subsection{Ablation Study}

We perform ablation study on ITC99 dataset and compare the performance of the ablation settings with our proposed method to evaluate the contribution of various components in our method. 
\begin{table}[h]
\centering
\small
\setlength{\tabcolsep}{3pt}
\vspace{-10pt}
\caption{Ablation Study on ITC99 Dataset($\%$).}

\scalebox{0.9}{
\begin{tabular}{@{}lcccccc@{}}
\toprule
\multicolumn{1}{c}{\multirow{3}{*}{Setting}} & \multicolumn{4}{c}{\textbf{Stage \#1}} & \multicolumn{2}{c}{\textbf{Stage \#2}} \\ \cmidrule(l){2-7} 
\multicolumn{1}{c}{} & \multicolumn{2}{c}{$\mathcal{G}_{sub} \to \mathcal{G}_{syn}$} & \multicolumn{2}{c}{$\mathcal{G}_{sub} \to \mathcal{G}_{pm}$} & \multicolumn{2}{c}{$\mathcal{G}_{sub} \to \mathcal{G}_{pm}$} \\ \cmidrule(l){2-7} 
\multicolumn{1}{c}{} & Accuracy & F1-score & Accuracy & F1-score & IoU & DICE \\ \midrule
Stage \#1\textit{ wo.} alignment & 94.6 & 94.6 & 91.4 & 91.5 & - & - \\ \midrule
Stage \#2 \textit{wo.} stage \#1 & - & - & - & - & 76.3 & 86.5 \\
Stage  \#2 \textit{wo.} seg. & - & - & - & - & 29.6 & 45.7 \\ \midrule
Ours & \textbf{95.3} & \textbf{95.4} & \textbf{93.1} & \textbf{93.2} & \textbf{83.0} & \textbf{90.7} \\ \bottomrule
\end{tabular}
}
\vspace{-10pt}
\end{table}

\textbf{Stage \#1 without alignment} achieves accuracy and F1-scores of $94.6\%$ and $94.6\%$ on $\mathcal{G}_{sub} \to \mathcal{G}_{syn}$ task, which are lower than our method’s $95.3\%$ and $95.4\%$. Our model also improves accuracy and F1-score by $1.7\%$ on $\mathcal{G}_{sub} \to \mathcal{G}_{pm}$ task. These results demonstrate the importance of function-invariant alignment, particularly inter-modal alignment, \ie aligning $\mathcal{G}_{pm}$ and $\mathcal{G}_{aig}$.

\textbf{Stage \#2 without Stage \#1} shows a performance drop, with IoU and DICE scores of $76.3\%$ and $86.5\%$, compared to our method’s improved values of $83.0\%$ and $90.7\%$. This highlights the crucial role of pretraining knowledge in Stage \#1.

\textbf{Stage \#2 without segmentation} also shows a significant drop in performance, with IoU and DICE values of $29.6\%$ and $45.7\%$, compared to our method’s improved $83.0\%$ and $90.7\%$. These results suggest that directly predicting the input-output nodes of the fuzzy boundary is challenging, as it varies with different functional transformations and omits the information of nodes in fuzzy boundary.

\section{Limitations}
\label{sec:limitation}
While our proposed framework demonstrates strong performance and significant improvements over existing structural approaches, several limitations remain and should be addressed in future research:

\textbf{Scalability to Large-scale Circuits:}
Currently, our method has primarily been evaluated on moderately-sized circuits due to computational resource constraints. Real-world EDA applications often involve extremely large netlists with millions of nodes. Scaling our detection and segmentation approaches to handle such large-scale graphs efficiently is non-trivial. Future research could investigate more computationally efficient embedding methods, hierarchical segmentation approaches, or incremental graph processing techniques to enhance scalability.

\textbf{Multiple and Overlapping Fuzzy Boundaries:}
Our fuzzy boundary identification method presently assumes a single, minimal enclosing region within the target graph. In practical scenarios, multiple occurrences or overlapping functional subgraphs might exist within a single large circuit, complicating boundary identification tasks. Extending our methodology to effectively handle multiple or overlapping fuzzy boundaries within the same circuit remains an open and challenging direction for further investigation.
 
\textbf{Single-output Circuit Assumption:}
The current approach assumes single-output logic circuits. In real-world scenarios, however, most circuits possess multiple outputs and complex internal functional dependencies. The direct applicability of our method to multi-output circuits, particularly when outputs share significant internal logic, remains unexplored. Generalizing the definitions and embedding strategies to model multi-output scenarios could further enhance practical relevance.

\textbf{Non-trivial Function Assumption:}
In this paper, we assume that a graph obtained by removing some nodes and edges is not functionally equivalent to the original graph, \ie $\forall g \neq \emptyset$, $G \setminus g \not\equiv_{\mathrm{func}} G$. While EDA tools inherently enforce this constraint, it may limit the generalizability of the functional subgraph in other domains.

By systematically addressing these limitations, subsequent research can extend our approach to broader, more realistic settings, thereby increasing its practical utility in EDA domain and beyond.

\section{Conclusion}
\label{sec:conclusion}
In this paper, we introduce the concept of \emph{functional subgraph matching}, a method to identify implicit logic functions within larger circuits, despite structural variations. We propose a two-stage framework: first, we train models across different modalities with alignment to detect functional subgraphs; second, we fine-tune our model and treat fuzzy boundary identification as a graph segmentation task for precise localization of fuzzy boundary. Evaluations on benchmarks (ITC99, OpenABCD, ForgeEDA) show that our approach outperforms structure-based methods, achieving $93.8\%$ accuracy in functional subgraph detection and a $91.3\%$ DICE score for fuzzy boundary detection.
\paragraph{Broader Impact} Our method contributes to the advancement of deep learning, particularly in graph-based functional relationship analysis. By improving the detection of functional relationships in complex systems, it has the potential to impact a wide range of applications, from circuit design to other domains that rely on graph functionality, \eg molecular and protein graphs.
\paragraph{Acknowledgment}
This work was supported in part by the Hong Kong Research Grants Council (RGC) under Grant No. 14212422, 14202824, and C6003-24Y.

\bibliographystyle{unsrtnat}
\bibliography{Subgraph}

\begin{thebibliography}{34}
\providecommand{\natexlab}[1]{#1}
\providecommand{\url}[1]{\texttt{#1}}
\expandafter\ifx\csname urlstyle\endcsname\relax
  \providecommand{\doi}[1]{doi: #1}\else
  \providecommand{\doi}{doi: \begingroup \urlstyle{rm}\Url}\fi

\bibitem[Mahzoon et~al.(2018)Mahzoon, Gro{\ss}e, and Drechsler]{mahzoon2018polycleaner}
Alireza Mahzoon, Daniel Gro{\ss}e, and Rolf Drechsler.
\newblock Polycleaner: clean your polynomials before backward rewriting to verify million-gate multipliers.
\newblock In \emph{2018 IEEE/ACM International Conference on Computer-Aided Design (ICCAD)}, pages 1--8. IEEE, 2018.

\bibitem[Mahzoon et~al.(2019)Mahzoon, Gro{\ss}e, and Drechsler]{mahzoon2019revsca}
Alireza Mahzoon, Daniel Gro{\ss}e, and Rolf Drechsler.
\newblock Revsca: Using reverse engineering to bring light into backward rewriting for big and dirty multipliers.
\newblock In \emph{Proceedings of the 56th Annual Design Automation Conference 2019}, pages 1--6, 2019.

\bibitem[Wei et~al.(2015)Wei, Diao, Lam, and Wu]{wei2015universal}
Xing Wei, Yi~Diao, Tak-Kei Lam, and Yu-Liang Wu.
\newblock A universal macro block mapping scheme for arithmetic circuits.
\newblock In \emph{2015 Design, Automation \& Test in Europe Conference \& Exhibition (DATE)}, pages 1629--1634. IEEE, 2015.

\bibitem[Meade et~al.(2016)Meade, Zhang, Jin, Zhao, and Pan]{meade2016gate}
Travis Meade, Shaojie Zhang, Yier Jin, Zheng Zhao, and David Pan.
\newblock Gate-level netlist reverse engineering tool set for functionality recovery and malicious logic detection.
\newblock In \emph{International Symposium for Testing and Failure Analysis}, volume 81368, pages 342--346. ASM International, 2016.

\bibitem[Li et~al.(2019)Li, Patnaik, Sengupta, Yang, Knechtel, Yu, Young, and Sinanoglu]{li2019attacking}
Haocheng Li, Satwik Patnaik, Abhrajit Sengupta, Haoyu Yang, Johann Knechtel, Bei Yu, Evangeline~FY Young, and Ozgur Sinanoglu.
\newblock Attacking split manufacturing from a deep learning perspective.
\newblock In \emph{Proceedings of the 56th Annual Design Automation Conference 2019}, pages 1--6, 2019.

\bibitem[Cordella et~al.(2004)Cordella, Foggia, Sansone, and Vento]{cordella2004sub}
Luigi~P Cordella, Pasquale Foggia, Carlo Sansone, and Mario Vento.
\newblock A (sub) graph isomorphism algorithm for matching large graphs.
\newblock \emph{IEEE transactions on pattern analysis and machine intelligence}, 26\penalty0 (10):\penalty0 1367--1372, 2004.

\bibitem[Ullmann(1976)]{ullmann1976algorithm}
Julian~R Ullmann.
\newblock An algorithm for subgraph isomorphism.
\newblock \emph{Journal of the ACM (JACM)}, 23\penalty0 (1):\penalty0 31--42, 1976.

\bibitem[Cordella et~al.(2001)Cordella, Foggia, Sansone, Vento, et~al.]{cordella2001improved}
Luigi~Pietro Cordella, Pasquale Foggia, Carlo Sansone, Mario Vento, et~al.
\newblock An improved algorithm for matching large graphs.
\newblock In \emph{3rd IAPR-TC15 workshop on graph-based representations in pattern recognition}, pages 149--159. Citeseer, 2001.

\bibitem[Bai et~al.(2019)Bai, Ding, Bian, Chen, Sun, and Wang]{bai2019simgnn}
Yunsheng Bai, Hao Ding, Song Bian, Ting Chen, Yizhou Sun, and Wei Wang.
\newblock Simgnn: A neural network approach to fast graph similarity computation.
\newblock In \emph{Proceedings of the twelfth ACM international conference on web search and data mining}, pages 384--392, 2019.

\bibitem[Lou et~al.(2020)Lou, You, Wen, Canedo, Leskovec, et~al.]{lou2020neural}
Zhaoyu Lou, Jiaxuan You, Chengtao Wen, Arquimedes Canedo, Jure Leskovec, et~al.
\newblock Neural subgraph matching.
\newblock \emph{arXiv preprint arXiv:2007.03092}, 2020.

\bibitem[Ying et~al.(2024)Ying, Fu, Wang, You, Wang, and Leskovec]{ying2024representation}
Rex Ying, Tianyu Fu, Andrew Wang, Jiaxuan You, Yu~Wang, and Jure Leskovec.
\newblock Representation learning for frequent subgraph mining.
\newblock \emph{arXiv preprint arXiv:2402.14367}, 2024.

\bibitem[Li et~al.(2024)Li, Wang, Chen, Sun, and Zhuo]{li2024efficient}
Bohao Li, Shizhang Wang, Tinghuan Chen, Qi~Sun, and Cheng Zhuo.
\newblock Efficient subgraph matching framework for fast subcircuit identification.
\newblock In \emph{Proceedings of the 2024 ACM/IEEE International Symposium on Machine Learning for CAD}, pages 1--7, 2024.

\bibitem[Wang et~al.(2022{\natexlab{a}})Wang, He, Bai, Yang, and Yu]{wang2022efficient}
Ziyi Wang, Zhuolun He, Chen Bai, Haoyu Yang, and Bei Yu.
\newblock Efficient arithmetic block identification with graph learning and network-flow.
\newblock \emph{IEEE Transactions on Computer-Aided Design of Integrated Circuits and Systems}, 42\penalty0 (8):\penalty0 2591--2603, 2022{\natexlab{a}}.

\bibitem[Wu et~al.(2023)Wu, Li, Hao, Dai, Yu, and Xie]{wu2023gamora}
Nan Wu, Yingjie Li, Cong Hao, Steve Dai, Cunxi Yu, and Yuan Xie.
\newblock Gamora: Graph learning based symbolic reasoning for large-scale boolean networks.
\newblock In \emph{2023 60th ACM/IEEE Design Automation Conference (DAC)}, pages 1--6. IEEE, 2023.

\bibitem[Davidson(1999)]{ITC99}
Scott Davidson.
\newblock Characteristics of the itc’99 benchmark circuits.
\newblock In \emph{ITSW}, 1999.

\bibitem[Chowdhury et~al.(2021)Chowdhury, Tan, Karri, and Garg]{chowdhury2021openabc}
Animesh~Basak Chowdhury, Benjamin Tan, Ramesh Karri, and Siddharth Garg.
\newblock Openabc-d: A large-scale dataset for machine learning guided integrated circuit synthesis.
\newblock \emph{arXiv preprint arXiv:2110.11292}, 2021.

\bibitem[Shi et~al.(2025{\natexlab{a}})Shi, Li, Ma, Zhou, Zheng, Liu, Pan, Zhou, Li, Zhu, et~al.]{shi2025forgeeda}
Zhengyuan Shi, Zeju Li, Chengyu Ma, Yunhao Zhou, Ziyang Zheng, Jiawei Liu, Hongyang Pan, Lingfeng Zhou, Kezhi Li, Jiaying Zhu, et~al.
\newblock Forgeeda: A comprehensive multimodal dataset for advancing eda.
\newblock \emph{arXiv preprint arXiv:2505.02016}, 2025{\natexlab{a}}.

\bibitem[Bonnici et~al.(2013)Bonnici, Giugno, Pulvirenti, Shasha, and Ferro]{bonnici2013subgraph}
Vincenzo Bonnici, Rosalba Giugno, Alfredo Pulvirenti, Dennis Shasha, and Alfredo Ferro.
\newblock A subgraph isomorphism algorithm and its application to biochemical data.
\newblock \emph{BMC bioinformatics}, 14:\penalty0 1--13, 2013.

\bibitem[Fan(2012)]{fan2012graph}
Wenfei Fan.
\newblock Graph pattern matching revised for social network analysis.
\newblock In \emph{Proceedings of the 15th international conference on database theory}, pages 8--21, 2012.

\bibitem[Kim et~al.(2015)Kim, Shin, Han, Hong, and Chafi]{kim2015taming}
Jinha Kim, Hyungyu Shin, Wook-Shin Han, Sungpack Hong, and Hassan Chafi.
\newblock Taming subgraph isomorphism for rdf query processing.
\newblock \emph{arXiv preprint arXiv:1506.01973}, 2015.

\bibitem[P{\'e}rez et~al.(2009)P{\'e}rez, Arenas, and Gutierrez]{perez2009semantics}
Jorge P{\'e}rez, Marcelo Arenas, and Claudio Gutierrez.
\newblock Semantics and complexity of sparql.
\newblock \emph{ACM Transactions on Database Systems (TODS)}, 34\penalty0 (3):\penalty0 1--45, 2009.

\bibitem[Shi et~al.(2023)Shi, Pan, Khan, Li, Liu, Huang, Zhen, Yuan, Chu, and Xu]{shi2023deepgate2}
Zhengyuan Shi, Hongyang Pan, Sadaf Khan, Min Li, Yi~Liu, Junhua Huang, Hui-Ling Zhen, Mingxuan Yuan, Zhufei Chu, and Qiang Xu.
\newblock Deepgate2: Functionality-aware circuit representation learning.
\newblock In \emph{2023 IEEE/ACM International Conference on Computer Aided Design (ICCAD)}, pages 1--9. IEEE, 2023.

\bibitem[Shi et~al.(2024)Shi, Zheng, Khan, Zhong, Li, and Xu]{shi2024deepgate3}
Zhengyuan Shi, Ziyang Zheng, Sadaf Khan, Jianyuan Zhong, Min Li, and Qiang Xu.
\newblock Deepgate3: Towards scalable circuit representation learning.
\newblock \emph{arXiv preprint arXiv:2407.11095}, 2024.

\bibitem[Zheng et~al.(2025)Zheng, Huang, Zhong, Shi, Dai, Xu, and Xu]{zheng2025deepgate4}
Ziyang Zheng, Shan Huang, Jianyuan Zhong, Zhengyuan Shi, Guohao Dai, Ningyi Xu, and Qiang Xu.
\newblock Deepgate4: Efficient and effective representation learning for circuit design at scale.
\newblock \emph{arXiv preprint arXiv:2502.01681}, 2025.

\bibitem[Shi et~al.(2025{\natexlab{b}})Shi, Ma, Zheng, Zhou, Pan, Jiang, Yang, Yang, Chu, and Xu]{shi2025deepcell}
Zhengyuan Shi, Chengyu Ma, Ziyang Zheng, Lingfeng Zhou, Hongyang Pan, Wentao Jiang, Fan Yang, Xiaoyan Yang, Zhufei Chu, and Qiang Xu.
\newblock Deepcell: Multiview representation learning for post-mapping netlists.
\newblock \emph{arXiv preprint arXiv:2502.06816}, 2025{\natexlab{b}}.

\bibitem[Liu et~al.(2024)Liu, Zhai, Zhao, Lin, Yu, and Shi]{PolarGate}
Jiawei Liu, Jianwang Zhai, Mingyu Zhao, Zhe Lin, Bei Yu, and Chuan Shi.
\newblock Polargate: Breaking the functionality representation bottleneck of and-inverter graph neural network.
\newblock In \emph{2024 IEEE/ACM International Conference on Computer-Aided Design (ICCAD)}, 2024.

\bibitem[Wang et~al.(2022{\natexlab{b}})Wang, Bai, He, Zhang, Xu, Ho, Yu, and Huang]{wang2022functionality}
Ziyi Wang, Chen Bai, Zhuolun He, Guangliang Zhang, Qiang Xu, Tsung-Yi Ho, Bei Yu, and Yu~Huang.
\newblock Functionality matters in netlist representation learning.
\newblock In \emph{Proceedings of the 59th ACM/IEEE Design Automation Conference}, pages 61--66, 2022{\natexlab{b}}.

\bibitem[Wang et~al.(2024)Wang, Bai, He, Zhang, Xu, Ho, Huang, and Yu]{wang2024fgnn2}
Ziyi Wang, Chen Bai, Zhuolun He, Guangliang Zhang, Qiang Xu, Tsung-Yi Ho, Yu~Huang, and Bei Yu.
\newblock Fgnn2: A powerful pre-training framework for learning the logic functionality of circuits.
\newblock \emph{IEEE Transactions on Computer-Aided Design of Integrated Circuits and Systems}, 2024.

\bibitem[He et~al.(2021)He, Wang, Bai, Yang, and Yu]{he2021graph}
Zhuolun He, Ziyi Wang, Chen Bai, Haoyu Yang, and Bei Yu.
\newblock Graph learning-based arithmetic block identification.
\newblock In \emph{2021 IEEE/ACM International Conference On Computer Aided Design (ICCAD)}, pages 1--8. IEEE, 2021.

\bibitem[Deng et~al.(2024)Deng, Yue, Yu, Sarar, Carey, Jain, and Zhang]{deng2024less}
Chenhui Deng, Zichao Yue, Cunxi Yu, Gokce Sarar, Ryan Carey, Rajeev Jain, and Zhiru Zhang.
\newblock Less is more: Hop-wise graph attention for scalable and generalizable learning on circuits.
\newblock In \emph{Proceedings of the 61st ACM/IEEE Design Automation Conference}, pages 1--6, 2024.

\bibitem[Brayton and Mishchenko(2010)]{brayton2010abc}
Robert Brayton and Alan Mishchenko.
\newblock Abc: An academic industrial-strength verification tool.
\newblock In \emph{CAV 2010, Edinburgh, UK, July 15-19, 2010. Proceedings 22}, pages 24--40. Springer, 2010.

\bibitem[Google()]{sky130nm}
Google.
\newblock Skywater open source pdk.
\newblock URL \url{https://github.com/google/skywater-pdk.git}.
\newblock 2020.

\bibitem[Oord et~al.(2018)Oord, Li, and Vinyals]{oord2018representation}
Aaron van~den Oord, Yazhe Li, and Oriol Vinyals.
\newblock Representation learning with contrastive predictive coding.
\newblock \emph{arXiv preprint arXiv:1807.03748}, 2018.

\bibitem[Carletti et~al.(2017)Carletti, Foggia, Saggese, and Vento]{vf3}
Vincenzo Carletti, Pasquale Foggia, Alessia Saggese, and Mario Vento.
\newblock Challenging the time complexity of exact subgraph isomorphism for huge and dense graphs with vf3.
\newblock \emph{IEEE transactions on pattern analysis and machine intelligence}, 40\penalty0 (4):\penalty0 804--818, 2017.

\end{thebibliography}


\appendix
\section{Proofs of the Proposed Properties}
\label{apdx:property}


In this section, we use $G_1\cong G_2$ to denote that $G_1$ is isomorphic to $G_2$. Also, we use $G_1\equiv_{\mathrm{func}}G_2$ to denote that $G_1$ is functional equivalent to $G_2$.


\begin{proposition}
If a graph $\mathcal{Q}$ is an equivalent subgraph of $\mathcal{G}$, then $\mathcal{Q}$ is a functional subgraph of $\mathcal{G}$.
\end{proposition}

\begin{proof}
    According to the Definition~\ref{def:subgraph_eq}, there exists a subgraph $\mathcal{G}'$ of $\mathcal{G}$ such that $\mathcal{Q} \equiv_{func} \mathcal{G}'$. By replacing $\mathcal{G}'$ with $\mathcal{Q}$, we get $\bar{\mathcal{G}} = \mathcal{G}\setminus\mathcal{G}' \cup \mathcal{Q}$ which is equivalent to $\mathcal{G}$ and a subgraph of $\mathcal{G}$ is isomorphic to $\mathcal{Q}$. By the Definition~\ref{def:subgraph_func}, $\mathcal{Q}$ is a functional subgraph of $\mathcal{G}$.
\end{proof}

\begin{proposition}[Reflexivity]
$\forall \mathcal{G}, \mathcal{G}\preccurlyeq\mathcal{G}.$
\end{proposition}
\begin{proof}
    $\mathcal{G}$ is a subgraph of itself, and $\mathcal{G} \equiv_{\mathrm{func}} \mathcal{G}$. By the definition of functional subgraph, it follows that $\mathcal{G} \preccurlyeq \mathcal{G}$.
\end{proof}


\begin{proposition}[Functional Equivalence Preservation]
    If $\mathcal{G}_1$ is a functional subgraph of $\mathcal{G}_2$, then:
        \begin{itemize}
            \item (Left-hand Side) if $\mathcal{G}_1\preccurlyeq\mathcal{G}_2$ and $\mathcal{G}'_1 \equiv_{func} \mathcal{G}_1$, then $\mathcal{G}'_1 \preccurlyeq \mathcal{G}_2$.
            \item (Right-hand Side) if $\mathcal{G}_1\preccurlyeq\mathcal{G}_2$ and $\mathcal{G}'_2 \equiv_{func} \mathcal{G}_2$, then $\mathcal{G}_1 \preccurlyeq \mathcal{G}'_2$.
        \end{itemize}
    \label{pro:eq_preserv}
\end{proposition}

\begin{proof} 
\textit{\textbf{(Right-hand Side)}} According to the Definition~\ref{def:subgraph_func}, if $\mathcal{G}_1$ is a functional subgraph of $\mathcal{G}_2$, then there exist $\mathcal{G}^{*}_2$ that $\mathcal{G}^{*}_2 \equiv_{\mathrm{func}} \mathcal{G}_2$ and $\mathcal{G}_1$ is an isomorphic subgraph of $\mathcal{G}^{*}_2$. Since $\mathcal{G}'_2 \equiv_{\mathrm{func}} \mathcal{G}_2$ and $\mathcal{G}^{*}_2 \equiv_{\mathrm{func}} \mathcal{G}_2$, then $\mathcal{G}^{*}_2$. Since $\mathcal{G}'_2 \equiv_{\mathrm{func}} \mathcal{G}'_2$. By the Definition~\ref{def:subgraph_func}, $\mathcal{G}_1$ is a functional subgraph of $\mathcal{G}'_2$.

\textit{\textbf{(Left-hand Side)}}
By Definition~\ref{def:subgraph_func}, there exists a graph $\mathcal{G}_2' \equiv_{\mathrm{func}} \mathcal{G}_2$, such that
\begin{equation}
\mathcal{G}_1 \cong \bar{\mathcal{G}}_2,
\end{equation}
where $\bar{\mathcal{G}}_2$ is a subgraph of $\mathcal{G}_2'$. Since $\mathcal{G}_1 \cong \bar{\mathcal{G}}_2$, it follows that
\begin{equation}
\mathcal{G}_1 \equiv_{\mathrm{func}} \bar{\mathcal{G}}_2.
\end{equation}
By the transitivity of functional equivalence, we then have
\begin{equation}
\mathcal{G}_1 \equiv_{\mathrm{func}} \bar{\mathcal{G}}_2 \equiv_{\mathrm{func}} \mathcal{G}_1'.
\end{equation}
Thus, by replacing $\bar{\mathcal{G}}_2$ in $\mathcal{G}_2'$ with $\mathcal{G}_1'$, we obtain a new graph
\begin{equation}
\mathcal{G}_2'' = (\mathcal{G}_2' \setminus \bar{\mathcal{G}}_2) \cup \mathcal{G}_1',
\end{equation}
which satisfies
\begin{equation}
\mathcal{G}_2'' \equiv_{\mathrm{func}} \mathcal{G}_2.
\end{equation}
From the definition of functional equivalence, we know that $\mathcal{G}_2'' \equiv_{\mathrm{func}} \mathcal{G}_2$ and that $\mathcal{G}_1'$ is a subgraph of $\mathcal{G}_2''$. Therefore, it follows that
\begin{equation}
\mathcal{G}_1' \preccurlyeq \mathcal{G}_2.
\end{equation}
\end{proof}

\begin{proposition}[Transitivity]
If $\mathcal{G}_1\preccurlyeq\mathcal{G}_2$ and $\mathcal{G}_2\preccurlyeq\mathcal{G}_3$, then $\mathcal{G}_1\preccurlyeq\mathcal{G}_3$.
\end{proposition}

\begin{proof}
By definition, there exists a graph $\mathcal{G}_2' \equiv_{\mathrm{func}} \mathcal{G}_2$, such that
\begin{equation}
\mathcal{G}_1 \cong \bar{\mathcal{G}}_1, \text{and } \bar{\mathcal{G}}_1 \text{is a subgraph of }\mathcal{G}_2'.
\end{equation} Since $\mathcal{G}_2' \equiv_{\mathrm{func}} \mathcal{G}_2$, by Proposition~\ref{pro:eq_preserv}, it follows that $\mathcal{G}_2' \preccurlyeq \mathcal{G}_3$. 

Therefore, there exists a graph $\mathcal{G}_3' \equiv_{\mathrm{func}} \mathcal{G}_3$, and $\mathcal{G}_2'$ is a subgraph of $\mathcal{G}_3'$. Since $\bar{\mathcal{G}}_1$ is a subgraph of $\mathcal{G}_2'$ and $\mathcal{G}_2'$ is a subgraph of $\mathcal{G}_3'$, it follows that $\bar{\mathcal{G}}_1$ is a subgraph of $\mathcal{G}_3'$.

Since $\mathcal{G}_1 \cong \bar{\mathcal{G}}_1$, $\mathcal{G}_3' \equiv_{\mathrm{func}} \mathcal{G}_3$ and $\bar{\mathcal{G}}_1$ is a subgraph of $\mathcal{G}_3'$, by the definition of functional subgraph, we conclude that
\begin{equation}
\mathcal{G}_1 \preccurlyeq \mathcal{G}_3.
\end{equation}
\end{proof}

\begin{proposition}[Anti‐symmetry]
$\mathcal{G}_1\preccurlyeq\mathcal{G}_2$ and $\mathcal{G}_2\preccurlyeq\mathcal{G}_1$ if and only if $\mathcal{G}_1\equiv_{\mathrm{func}}\mathcal{G}_2$.
\end{proposition}
\begin{proof}
($\Rightarrow$) Since $\mathcal{G}_1 \preccurlyeq \mathcal{G}_2$, we have 
\begin{equation}
    \mathcal{G}_1 \cong \mathcal{G}_2' \setminus g, \quad \text{and} \quad \mathcal{G}_2 \equiv_{\mathrm{func}} \mathcal{G}_2'.
\end{equation}
Since $\mathcal{G}_2 \preccurlyeq \mathcal{G}_1$, we have $\mathcal{G}_2' \preccurlyeq \mathcal{G}_2' \setminus g$. By the definition of functional subgraphs, there exists a graph $\mathcal{G}_3$ such that $\mathcal{G}_3 \equiv_{\mathrm{func}} \mathcal{G}_2' \setminus g$ and $\mathcal{G}_2'$ is a subgraph of $\mathcal{G}_3$. This implies that $\mathcal{G}_2' \cong \mathcal{G}_3 \setminus g'$, so we also have
\begin{equation}
\mathcal{G}_2' \equiv_{\mathrm{func}} \mathcal{G}_3 \setminus g'.
\end{equation}
Since $\mathcal{G}_3 \equiv_{\mathrm{func}} \mathcal{G}_2' \setminus g$, it follows that
\begin{equation}
    \mathcal{G}_3 \cup g \equiv_{\mathrm{func}} \mathcal{G}_2'.
\end{equation}
Thus, we have
\begin{equation}
    \mathcal{G}_3 \cup g \equiv_{\mathrm{func}} \mathcal{G}_2' \equiv_{\mathrm{func}} \mathcal{G}_3 \setminus g'.
\end{equation}
Note that in Section~\ref{sec:property}, we assume that a graph obtained by removing some nodes and edges is not functionally equivalent to the original graph, i.e., $\forall g \neq \emptyset$, $G \setminus g \not\equiv_{\mathrm{func}} G$. Therefore, we must have $g = g' = \emptyset$, which implies
\begin{equation}
    \mathcal{G}_1 \cong \mathcal{G}_2' \setminus g \cong \mathcal{G}_2', \quad \text{and} \quad \mathcal{G}_2 \equiv_{\mathrm{func}} \mathcal{G}_2'.
\end{equation}
Thus, we conclude that
\begin{equation}
    \mathcal{G}_1 \equiv_{\mathrm{func}} \mathcal{G}_2.
\end{equation}

($\Leftarrow$) If $\mathcal{G}_1\equiv_{\mathrm{func}}\mathcal{G}_2$, since $\mathcal{G}_1\preccurlyeq\mathcal{G}_1$ and $\mathcal{G}_2\preccurlyeq\mathcal{G}_2$, according to \textit{Functional Equivalence Preservation} property, it follows that $\mathcal{G}_1\preccurlyeq\mathcal{G}_2$ and $\mathcal{G}_2\preccurlyeq\mathcal{G}_1$.
\end{proof}

\section{Additional Experimental Results}
\label{apdx:other_baseline}
\subsection{Functional Subgraph Matching} 
Considering that the encoder in our method can be replaced with other backbones, we test our approach with different encoders and propose baselines for the functional subgraph detection task, as shown in Table~\ref{tab:stage1_baseline}.

\begin{table}[h]
\small
\setlength{\tabcolsep}{3pt}
\centering
\caption{Result of baselines in stage \#1.}

\begin{tabular}{@{}cccccccccc@{}}
\toprule
\multirow{2}{*}{Dataset} & \multirow{2}{*}{Method} & \multicolumn{4}{c}{$\mathcal{G}_{sub} \to \mathcal{G}_{syn}$} & \multicolumn{4}{c}{$\mathcal{G}_{sub} \to \mathcal{G}_{pm}$} \\ \cmidrule(l){3-10} 
 &  & Accuracy & Precision & Recall & F1-score & Accuracy & Precision & Recall & F1-score \\ \midrule
\multirow{3}{*}{ITC99} & Ours+Gamora & 90.8 & 91.1 & 90.4 & 90.7 & 86.4 & 88.6 & 83.5 & 86.0 \\
 & Ours+ABGNN & 87.9 & 83.1 & 95.1 & 88.7 & 88.2 & 82.8 & \textbf{96.5} & 89.1 \\
 & Ours & \textbf{95.3} & \textbf{94.4} & \textbf{96.3} & \textbf{95.4} & \textbf{93.1} & \textbf{92.3} & 94.2 & \textbf{93.2} \\ \midrule
\multirow{3}{*}{OpenABCD} & Ours+Gamora & 90.1 & 89.6 & \textbf{90.7} & 90.2 & \textbf{91.0} & 89.3 & \textbf{93.2} & \textbf{91.2} \\
 & Ours+ABGNN & 81.7 & 78.5 & 87.5 & 82.7 & 83.3 & 78.9 & 91.1 & 84.5 \\
 & Ours & \textbf{92.3} & \textbf{93.7} & 90.6 & \textbf{92.1} & 90.8 & \textbf{92.4} & 88.9 & 90.6 \\ \midrule
\multirow{3}{*}{ForgeEDA} & Ours+Gamora & 94.2 & 95.9 & 92.4 & 94.1 & 80.6 & 93.8 & 65.5 & 77.1 \\
 & Ours+ABGNN & 89.7 & 88.5 & 91.2 & 89.8 & 87.6 & 88.3 & 86.8 & 87.5 \\
 & Ours & \textbf{96.0} & \textbf{96.8} & \textbf{95.2} & \textbf{96.0} & \textbf{95.3} & \textbf{95.9} & \textbf{94.7} & \textbf{95.3} \\ \bottomrule
\end{tabular}

\label{tab:stage1_baseline}
\end{table}

\subsection{Fuzzy Boundary Identification}
We further evaluate these methods on fuzzy boundary identification. The results are shown in Table~\ref{tab:stage2_baseline}.

\begin{table}[H]
\caption{Result of baselines in stage \#2.}
\small
\centering
\begin{tabular}{@{}ccccccc@{}}
\toprule
\multirow{2}{*}{Method} & \multicolumn{2}{c}{ITC99} & \multicolumn{2}{c}{OpenABCD} & \multicolumn{2}{c}{ForgeEDA} \\ \cmidrule(l){2-7} 
 & IoU & DICE & IoU & DICE & IoU & DICE \\ \midrule
Ours+Gamora & 82.1 & 90.2 & 81.4 & 89.8 & 83.6 & 91.1 \\
Ours+ABGNN & 82.7 & 90.5 & 84.4 & 91.5 & \textbf{88.4} & \textbf{93.8} \\
Ours & \textbf{83.0} & \textbf{90.7} & \textbf{85.2} & \textbf{92.0} & 83.8 & 91.2 \\ \bottomrule
\end{tabular}
\label{tab:stage2_baseline}
\end{table}

\subsection{Scalability on Medium-Sized Circuits}
we collect a medium-sized graph dataset from ForgeEDA~\cite{shi2025forgeeda}, containing circuits with graph sizes ranging from 100 to 10000 nodes. The statistical information of the medium-sized dataset is shown in Table~\ref{tab:med_dataset}. 
\begin{table}[htbp]
\centering
\caption{Statistics of the medium-sized dataset.}
\label{tab:med_dataset}
\begin{tabular}{lccc}
\toprule
\textbf{Graph Type} & \textbf{Nodes} & \textbf{Edges} & \textbf{Depth} \\
\midrule
$G_{sub}$ & $192.1 \pm 320.87$  & $207.16 \pm 354.01$ & $27.91 \pm 30.08$  \\
$G_{syn}$ & $1396.99 \pm 1764.63$ & $1958.84 \pm 2519.23$ & $66.48 \pm 99.7$   \\
$G_{pm}$  & $679.93 \pm 851.2$  & $1352.96 \pm 1703.16$ & $21.83 \pm 32.6$   \\
\bottomrule
\end{tabular}
\end{table}

We perform functional subgraph detection on this dataset, and the results are shown in Table~\ref{tab:med_functional_subgraph} and ~\ref{tab:med_functional_subgraph_pm}. Since ABGNN encounters out-of-memory error when encoding graphs with deep logic levels, we do not report its results on this new dataset. While our method still demonstrates state-of-the-art performance, it shows a significant performance drop (from an F1-score of 95.3\% to 81.2\%, as shown in Table 1). This result highlights the challenge of scaling to larger circuits. We hope future work will explore and address this challenge.

\begin{table}[htbp]
\centering
\caption{Functional Subgraph Detection on $G_{sub} \to G_{syn}$.}
\label{tab:med_functional_subgraph}
\begin{tabular}{lcccc}
\toprule
\textbf{Method} & \textbf{Accuracy} & \textbf{Precision} & \textbf{Recall} & \textbf{F1-score} \\
\midrule
NeuroMatch & $51.2_{\pm3.3}$ & $34.6_{\pm24.5}$ & $66.7_{\pm47.1}$ & $45.5_{\pm32.2}$ \\
HGCN       & $50.0_{\pm0.0}$ & $16.7_{\pm23.6}$ & $33.3_{\pm47.1}$ & $22.2_{\pm31.4}$ \\
Gamora     & $50.0_{\pm0.0}$ & $40.0_{\pm14.1}$ & $66.7_{\pm47.1}$ & $44.6_{\pm31.3}$ \\
\textbf{Ours} & $\mathbf{81.5_{\pm0.6}}$ & $\mathbf{82.7_{\pm1.1}}$ & $\mathbf{79.8_{\pm1.7}}$ & $\mathbf{81.2_{\pm0.8}}$ \\
\bottomrule
\end{tabular}
\end{table}

\begin{table}[htbp]
\centering
\caption{Functional Subgraph Detection on $G_{sub} \to G_{pm}$.}
\label{tab:med_functional_subgraph_pm}
\begin{tabular}{lcccc}
\toprule
\textbf{Method} & \textbf{Accuracy} & \textbf{Precision} & \textbf{Recall} & \textbf{F1-score} \\
\midrule
NeuroMatch & $51.0_{\pm1.2}$ & $33.9_{\pm24.0}$ & $66.6_{\pm47.1}$ & $44.9_{\pm31.8}$ \\
HGCN       & $50.0_{\pm0.0}$ & $16.7_{\pm23.6}$ & $33.3_{\pm47.1}$ & $22.2_{\pm31.4}$ \\
Gamora     & $50.0_{\pm0.0}$ & $33.3_{\pm23.6}$ & $66.7_{\pm47.1}$ & $44.4_{\pm31.4}$ \\
\textbf{Ours} & $\mathbf{78.9_{\pm1.0}}$ & $\mathbf{80.6_{\pm1.6}}$ & $\mathbf{76.3_{\pm2.7}}$ & $\mathbf{78.3_{\pm1.3}}$ \\
\bottomrule
\end{tabular}
\end{table}

\subsection{Comparison with VF3}
we evaluate the state-of-the-art subgraph isomorphic heuristic algorithm VF3~\cite{vf3}, which consistently achieves 100\% precision and 100\% recall on standard subgraph isomorphism tasks. Due to time constraints, we sampled circuits with fewer than 50 nodes from the ForgeEDA dataset and applied the VF3 algorithm. The results are shown in Table~\ref{tab:subgraph_isomorphism_runtime}.
According to our Definition~\ref{def:subgraph_func} of Functional Subgraph, if $Q$ is an isomorphic subgraph of $G$, then $Q$ is always a functional subgraph of $G$. This is demonstrated by the 100\% precision achieved by VF3. However, due to the function-preserving transformation, the explicit structure of $Q$ often disappears, leading to extremely low recall(0.32\%) for VF3. These results highlight the importance of task definition. 

\begin{table*}[htbp]
\centering
\caption{Comparison of subgraph isomorphism methods on different tasks.}
\label{tab:subgraph_isomorphism_runtime}
\begin{tabular}{lrrrrrrr}
\toprule
& & \multicolumn{3}{c}{\textbf{$G_{sub} \to G_{syn}$}} & \multicolumn{3}{c}{\textbf{$G_{sub} \to G_{pm}$}} \\
\cmidrule(lr){3-5} \cmidrule(lr){6-8}
\textbf{Method} & \textbf{Runtime (s)} & Precision & Recall & F1-score & Precision & Recall & F1-score \\
\midrule
VF3  & 480.8 & \textbf{100.0} & 0.32 & 0.65 & --- & --- & --- \\
\textbf{Ours} & \textbf{8.0} & 88.5 & \textbf{91.9} & \textbf{90.2} & \textbf{86.4} & \textbf{94.5} & \textbf{90.3} \\
\bottomrule
\end{tabular}
\end{table*}

\section{Datasets and Implementation Details}
\label{apdx:dataset}
\paragraph{Dataset}
Dataset statistics and splits are shown in Table ~\ref{tab:data}. For dataset split, we first split the training circuits and test circuits in the source dataset, then we cut subgraph for the training circuit and test circuits to generate our small circuit dataset. For ITC99 and OpenABCD, the split follow the previous work~\cite{zheng2025deepgate4}. For ForgeEDA, we randomly select $10\%$ circuits in the dataset as test circuits. For small circuit, we apply Algorithm~\ref{alg:bfs} to randomly sample subgraph.

\begin{algorithm}[h]
\renewcommand{\algorithmicrequire}{\textbf{Input:}}
\renewcommand{\algorithmicensure}{\textbf{Output:}}
\caption{Random Sample Subgraph}
\centering
\begin{algorithmic}[1.0]
\REQUIRE ndoes $\mathcal{V}$, edges $\mathcal{E}$, root $r$
\ENSURE nodes $V$, edges $E$, root $r$
\STATE Build adjacency $G$ from $\mathcal{E}$
\IF{$rand(0,1)<0.5$} 
\STATE Set $r$ to the predecessor $p\in G[r]$ that maximizes $\mathrm{predCount}(p)$
\ENDIF
\STATE $\rho \gets rand(0.6,0.95)$
\STATE $T\gets\rho\cdot|\mathcal{V}|$, $V\gets\{r\}$, $Q\gets[r]$, $E\gets\emptyset$
\WHILE{$Q\neq\emptyset\wedge|V|<T$}
  \STATE $n\gets\mathrm{pop}(Q)$
  \FORALL{$v\in\mathrm{shuffle}(G[n])$}
    \IF{$v\notin V$}
      \STATE push($Q,v$); $V\cup\{v\}$; $E\cup\{(n,v)\}$
    \ENDIF
  \ENDFOR
\ENDWHILE
\RETURN $V,E,r$
\end{algorithmic}
\label{alg:bfs}
\end{algorithm}


\begin{table}[H]
\setlength{\tabcolsep}{3pt}
\centering
\caption{Dataset Statistics. We report average and standard error with $avg.\pm std.$}
\scalebox{0.9}{
\begin{tabular}{@{}ccccccccccc@{}}
\toprule
\multirow{2}{*}{\textbf{\begin{tabular}[c]{@{}c@{}}Source \\ Dataset\end{tabular}}} & \multirow{2}{*}{\textbf{Split}} & \multirow{2}{*}{\textbf{\#Pair}} & \multicolumn{2}{c}{\textbf{$\mathcal{G}_{sub}$}} & \multicolumn{2}{c}{\textbf{$\mathcal{G}_{aig}$}} & \multicolumn{2}{c}{\textbf{$\mathcal{G}_{syn}$}} & \multicolumn{2}{c}{\textbf{$\mathcal{G}_{pm}$}} \\ \cmidrule(l){4-11} 
 &  &  & \textbf{\#Node} & \textbf{Depth} & \textbf{\#Node} & \textbf{Depth} & \textbf{\#Node} & \textbf{Depth} & \textbf{\#Node} & \textbf{Depth} \\ \midrule
\multirow{2}{*}{ITC99} & train & 36592 & $248_{\pm 132}$ & 15.0$_{\pm 2.0}$ & $320_{\pm 166}$ & 19.1$_{\pm 3.0}$ & $315_{\pm 164}$ & 19.0$_{\pm 3.0}$ & $179_{\pm 91}$ & 6.9$_{\pm 1.0}$ \\
 & test & 5917 & $218_{\pm 113}$ & 14.0$_{\pm 2.0}$ & $282_{\pm 141}$ & 17.3$_{\pm 2.2}$ & $278_{\pm 138}$ & 17.0$_{\pm 2.0}$ & $157_{\pm 79}$ & 6.3$_{\pm 0.9}$ \\ \midrule
\multirow{2}{*}{OpenABCD} & train & 54939 & $155_{\pm 113}$ & 13.0$_{\pm 2.0}$ & $203_{\pm 140}$ & 16.4$_{\pm 3.2}$ & $198_{\pm 134}$ & 16.0$_{\pm 3.0}$ & $108_{\pm 75}$ & 5.8$_{\pm 1.1}$ \\
 & test & 9726 & $100_{\pm 66}$ & 13.0$_{\pm 2.0}$ & $132_{\pm 84}$ & 16.0$_{\pm 2.2}$ & $128_{\pm 82}$ & 15.0$_{\pm 2.0}$ & $69_{\pm 46}$ & 5.5$_{\pm 0.9}$ \\ \midrule
\multirow{2}{*}{ForgeEDA} & train & 60183 & $126_{\pm 102}$ & 13.4$_{\pm 3.5}$ & $161_{\pm 129}$ & 16.6$_{\pm 4.2}$ & $156_{\pm 125}$ & 16.2$_{\pm 4.5}$ & $88_{\pm 69}$ & 5.8$_{\pm 1.4}$ \\
 & test & 7753 & $127_{\pm 96}$ & 13.6$_{\pm 3.3}$ & $163_{\pm 122}$ & 17.0$_{\pm 3.8}$ & $159_{\pm 120}$ & 16.4$_{\pm 4.2}$ & $89_{\pm 65}$ & 5.9$_{\pm 1.3}$ \\ \bottomrule
\end{tabular}
}
\label{tab:data}
\end{table}

\paragraph{Environment}
All experiments are run on an NVIDIA A100 GPU with 64\,GB of memory. Models are trained using the Adam optimizer with a learning rate of 0.001, a batch size of 1024. We train our model in stage\#1 for 100 epochs and finetune it in stage\#2 for 10 epochs. Training our model on one dataset takes approximately 10 hours. Model architectures follow the configurations specified in the original works except that we set the hidden dimension to 128 for all models.

\paragraph{Evaluation Metrics}
For Stage \#1, we measure classification performance by accuracy and report precision, recall and f1-score according to the counts of true positives (TP), true negatives (TN), false positives (FP), and false negatives (FN):
\begin{align*}
\mathrm{Precision}
= \frac{\mathrm{TP}}{\mathrm{TP} + \mathrm{FP}}&, \  
\mathrm{Recall}
= \frac{\mathrm{TP}}{\mathrm{TP} + \mathrm{FN}}, \\
\mathrm{Accuracy}
= \frac{\mathrm{TP} + \mathrm{TN}}{\mathrm{TP} + \mathrm{TN} + \mathrm{FP} + \mathrm{FN}}&, \ 
\mathrm{F1\text{-}score}
= \frac{2 \times \mathrm{Precision} \times \mathrm{Recall}}{\mathrm{Precision} + \mathrm{Recall}}.
\end{align*}
For Stage \#2, which is similar to a segmentation task, we use Intersection over Union (IoU) and the Dice coefficient. Let \(P\) be the set of predicted positive nodes and \(G\) the set of ground-truth positive nodes:
\begin{equation}
\mathrm{IoU}
= \frac{\lvert P \cap G \rvert}{\lvert P \cup G \rvert}, \ \mathrm{Dice}
= \frac{2\,\lvert P \cap G \rvert}{\lvert P \rvert + \lvert G \rvert}
\end{equation}

\section{Background}
\label{apdx:background}
\paragraph{And-Inverter Graph(AIG)} In our works, AIG is a directed acyclic graph (DAG) composed of three basic elements: Primary Input(PI), AND gate and NOT gate. For example, a simple logic expression $\lnot A\land B$ can be build as a DAG with 2 PIs(A and B), one NOT gate and one AND gate. The edges are $[(A,NOT), (NOT,AND),(B,AND)]$. Since the out-degree of $AND$ is zero, it represents the final output.
\paragraph{Technology Mapping} Technology Mapping is a function-preserving transformation that converts an AIG into a post-mapping (PM) netlist. While the AIG consists of simple logic elements, such as AND and NOT gates, the basic components in a PM netlist can be more complex, such as adders or multipliers. As a result, node types can differ significantly between the two forms, and this is why we consider AIG and PM netlists as distinct modalities in this paper.
\paragraph{Logic Synthesis} Logic synthesis aims to simplify the structure of an AIG while preserving its functionality. It transforms one AIG into another with a simpler structure. For example, the expression $eq_1: (A\land B)\land(A\land C)$ can be simplified to $eq_2:A\land B \land C$.  Although $eq_1$ and $eq_2$ are functional equivalent, $eq_2$ uses only 2 AND gates compared to 3 AND gates in $eq_1$.
\paragraph{InfoNCE Loss} InfoNCE (Information Noise-Contrastive Estimation) is a contrastive loss used in self-supervised learning.
Its goal is to identify a single ``positive'' sample from a set of ``negative'' samples for a given ``anchor'' sample. It pulls the anchor and positive representations together while pushing the anchor and negatives apart:
\begin{equation}
    \mathcal{L}_{InfoNCE} = -\log \frac{\exp(\text{sim}(q, k_+) / \tau)}{\sum_{i=0}^{N} \exp(\text{sim}(q, k_i) / \tau)}
\end{equation}
where $q$ is the anchor, $k_+$ is the positive, $k_i$ are the negatives, $\text{sim}$ is a similarity function (like dot product), and $\tau$ is a temperature hyperparameter.


\clearpage

\section*{NeurIPS Paper Checklist}

\begin{enumerate}

\item {\bf Claims}
    \item[] Question: Do the main claims made in the abstract and introduction accurately reflect the paper's contributions and scope?
    \item[] Answer: \answerYes{} 
    \item[] Justification: We list our contribution in introduction, as shown in Section~\ref{sec:intro}.

\item {\bf Limitations}
    \item[] Question: Does the paper discuss the limitations of the work performed by the authors?
    \item[] Answer: \answerYes{} 
    \item[] Justification: We discuss limitations of our method in Section~\ref{sec:limitation}, including the Scalability to Large-scale Circuits, Multiple and Overlapping Fuzzy Boundaries and Assumptions.

\item {\bf Theory assumptions and proofs}
    \item[] Question: For each theoretical result, does the paper provide the full set of assumptions and a complete (and correct) proof?
    \item[] Answer: \answerYes{} 
    \item[] Justification: We proposed 4 properties of functional subgraph in our paper, and the proof can be found in Appendix~\ref{apdx:property}. We also clairfy our assumption in Section~\ref{sec:property}.

    \item {\bf Experimental result reproducibility}
    \item[] Question: Does the paper fully disclose all the information needed to reproduce the main experimental results of the paper to the extent that it affects the main claims and/or conclusions of the paper (regardless of whether the code and data are provided or not)?
    \item[] Answer: \answerYes{} 
    \item[] Justification: For dataset processing, we provide the details in Section~\ref{sec:data} and Appendix~\ref{apdx:dataset}. For our models, we first detailed our pipeline in Section~\ref{sec:method}.
    Futhermore, our used backbones are all opensource and the implementations details can also be found in Section~\ref{sec:data} and Appendix~\ref{apdx:dataset}.

\item {\bf Open access to data and code}
    \item[] Question: Does the paper provide open access to the data and code, with sufficient instructions to faithfully reproduce the main experimental results, as described in supplemental material?
    \item[] Answer: \answerYes{} 
    \item[] Justification: We provide the code and data in our supplemental material.

\item {\bf Experimental setting/details}
    \item[] Question: Does the paper specify all the training and test details (e.g., data splits, hyperparameters, how they were chosen, type of optimizer, etc.) necessary to understand the results?
    \item[] Answer: \answerYes{} 
    \item[] Justification: We specify the experimental details in Section~\ref{sec:data} and Appendix~\ref{apdx:dataset}.

\item {\bf Experiment statistical significance}
    \item[] Question: Does the paper report error bars suitably and correctly defined or other appropriate information about the statistical significance of the experiments?
    \item[] Answer: \answerYes{} 
    \item[] Justification: We report 1-sigma error bars, \ie average result with standard error in Table~\ref{tab:stage1} and ~\ref{tab:stage2}.

\item {\bf Experiments compute resources}
    \item[] Question: For each experiment, does the paper provide sufficient information on the computer resources (type of compute workers, memory, time of execution) needed to reproduce the experiments?
    \item[] Answer: \answerYes{} 
    \item[] Justification: The detail about compute resources can be found in Appendix~\ref{apdx:dataset}.
    
\item {\bf Code of ethics}
    \item[] Question: Does the research conducted in the paper conform, in every respect, with the NeurIPS Code of Ethics \url{https://neurips.cc/public/EthicsGuidelines}?
    \item[] Answer: \answerYes{} 
    \item[] Justification: We make sure research conducted in the paper conform with the NeurIPS Code of Ethics.

\item {\bf Broader impacts}
    \item[] Question: Does the paper discuss both potential positive societal impacts and negative societal impacts of the work performed?
    \item[] Answer: \answerYes{} 
    \item[] Justification: We discuss the broader impacts in Section~\ref{sec:conclusion}.
    
\item {\bf Safeguards}
    \item[] Question: Does the paper describe safeguards that have been put in place for responsible release of data or models that have a high risk for misuse (e.g., pretrained language models, image generators, or scraped datasets)?
    \item[] Answer: \answerNA{} 
    \item[] Justification: The paper poses no such risks.

\item {\bf Licenses for existing assets}
    \item[] Question: Are the creators or original owners of assets (e.g., code, data, models), used in the paper, properly credited and are the license and terms of use explicitly mentioned and properly respected?
    \item[] Answer: \answerYes{} 
    \item[] Justification: We cite the original paper or website of assets.

\item {\bf New assets}
    \item[] Question: Are new assets introduced in the paper well documented and is the documentation provided alongside the assets?
    \item[] Answer: \answerYes{} 
    \item[] Justification: We detail our dataset processing flow and model in our paper.

\item {\bf Crowdsourcing and research with human subjects}
    \item[] Question: For crowdsourcing experiments and research with human subjects, does the paper include the full text of instructions given to participants and screenshots, if applicable, as well as details about compensation (if any)? 
    \item[] Answer: \answerNA{} 
    \item[] Justification: Our paper does not involve crowdsourcing nor research with human subjects

\item {\bf Institutional review board (IRB) approvals or equivalent for research with human subjects}
    \item[] Question: Does the paper describe potential risks incurred by study participants, whether such risks were disclosed to the subjects, and whether Institutional Review Board (IRB) approvals (or an equivalent approval/review based on the requirements of your country or institution) were obtained?
    \item[] Answer: \answerNA{} 
    \item[] Justification: Our paper does not involve crowdsourcing nor research with human subjects.

\item {\bf Declaration of LLM usage}
    \item[] Question: Does the paper describe the usage of LLMs if it is an important, original, or non-standard component of the core methods in this research? Note that if the LLM is used only for writing, editing, or formatting purposes and does not impact the core methodology, scientific rigorousness, or originality of the research, declaration is not required.
    \item[] Answer: \answerNA{} 
    \item[] Justification: The core method development in this research does not involve LLMs as any important, original, or non-standard components.

\end{enumerate}

\end{document}